\documentclass{article}

\usepackage{microtype}
\usepackage{graphicx}
\usepackage{subfigure}
\usepackage{booktabs} % for professional tables

\usepackage{hyperref}

\usepackage[accepted]{icml2019}

\usepackage{amsmath}
\usepackage{amsfonts}       % blackboard math symbols
\usepackage{dsfont}
\graphicspath{{Images/}}
\usepackage{color}  % TODO: remove this before submission
\usepackage{mathtools}
\usepackage{mathrsfs}
\mathtoolsset{showonlyrefs}

\usepackage{amsthm}

\newtheorem{thm}{Theorem}[]

\newtheorem{lemma}{Lemma}[]

\newtheorem{assumpA}{Assumption}
\newtheorem{assumpB}{Assumption}

\usepackage[vlined,linesnumbered,ruled,algo2e]{algorithm2e}
\SetKwProg{Fn}{Function}{}{}

\icmltitlerunning{Learning Action Representations for Reinforcement Learning}

\begin{document}

\twocolumn[
\icmltitle{Learning Action Representations for Reinforcement Learning}

% It is OKAY to include author information, even for blind
% submissions: the style file will automatically remove it for you
% unless you've provided the [accepted] option to the icml2018
% package.

% List of affiliations: The first argument should be a (short)
% identifier you will use later to specify author affiliations
% Academic affiliations should list Department, University, City, Region, Country
% Industry affiliations should list Company, City, Region, Country

% You can specify symbols, otherwise they are numbered in order.
% Ideally, you should not use this facility. Affiliations will be numbered
% in order of appearance and this is the preferred way.
\icmlsetsymbol{equal}{*}

\begin{icmlauthorlist}
\icmlauthor{Yash Chandak}{umass}
\icmlauthor{Georgios Theocharous}{adobe}
\icmlauthor{James E. Kostas}{umass}
\icmlauthor{Scott M. Jordan}{umass}
\icmlauthor{Philip S. Thomas}{umass}
\end{icmlauthorlist}

\icmlaffiliation{umass}{University of Massachusetts, Amherst, USA.}
\icmlaffiliation{adobe}{Adobe Research, San Jose, USA.}

\icmlcorrespondingauthor{Yash Chandak}{ychandak@cs.umass.edu}

% You may provide any keywords that you
% find helpful for describing your paper; these are used to populate
% the "keywords" metadata in the PDF but will not be shown in the document
\icmlkeywords{Machine Learning, ICML}

\vskip 0.3in
]

% this must go after the closing bracket ] following \twocolumn[ ...

% This command actually creates the footnote in the first column
% listing the affiliations and the copyright notice.
% The command takes one argument, which is text to display at the start of the footnote.
% The \icmlEqualContribution command is standard text for equal contribution.
% Remove it (just {}) if you do not need this facility.

\printAffiliationsAndNotice{}  % leave blank if no need to mention equal contribution
%\printAffiliationsAndNotice{\icmlEqualContribution} % otherwise use the standard text.

\begin{abstract}
	    Most model-free reinforcement learning methods leverage state representations (embeddings) for generalization, but either ignore structure in the space of actions or assume the structure is provided \emph{a priori}.
	    We show how a policy can be decomposed into a component that acts in a low-dimensional space of action representations and a component that transforms these representations into actual actions.
	    These representations improve generalization over large, finite action sets by allowing the agent to infer the outcomes of actions similar to actions already taken. 
	    We provide an algorithm to both learn and use action representations and provide conditions for its convergence.
	    The efficacy of the proposed method is demonstrated on large-scale real-world problems.
\end{abstract}

	\section{Introduction}

\emph{Reinforcement learning} (RL) methods have been applied successfully to many simple and game-based tasks. 
However, their applicability is still limited for problems involving decision making in many real-world settings. One reason is that many real-world problems with significant human impact involve selecting a single decision from a multitude of possible choices. 
For example, maximizing long-term portfolio value in finance using various trading strategies \cite{jiang2017deep}, improving fault tolerance by regulating voltage level of all the units in a large power system \cite{glavic2017reinforcement}, and personalized tutoring systems for recommending sequences of videos from a large collection of tutorials \cite{sidney2005integrating}. %optimal parameter configuration of web systems to handle traffic and enhance user experience \cite{jiang2017deep}, and optimizing chemical reactions by controlling all the possible ways to regulate the reaction \cite{zhou2017optimizing}.   
Therefore, it is important that we develop RL algorithms that are effective for real problems, where the number of possible choices is large.

In this paper we consider the problem of creating RL algorithms that are effective for problems with large action sets. 
Existing RL algorithms handle large \emph{state} sets (e.g., images consisting of pixels) by learning a representation or embedding for states (e.g., using line detectors or convolutional layers in neural networks), which allow the agent to reason and learn using the state representation rather than the raw state. 
We extend this idea to the set of actions: we propose learning a representation for the actions, which allows the agent to reason and learn by making decisions in the space of action representations rather than the original large set of possible actions. 
This setup is depicted in Figure \ref{Fig:execution}, where an \emph{internal policy}, $\pi_i$, acts in a space of action representations, and a function, $f$, transforms these representations into actual actions. 
Together we refer to $\pi_i$ and $f$ as the \emph{overall policy}, $\pi_o$.

\begin{figure}[t]
    \centering
	\includegraphics[scale=0.2]{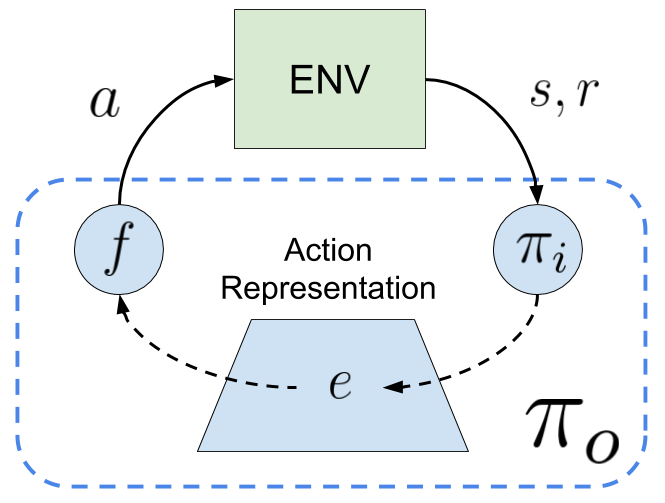}
	\caption{ 
	The structure of the proposed overall policy, $\pi_o$, consisting of $f$ and $\pi_i$, that learns action representations to generalize over large action sets. 
	}
	\label{Fig:execution}
\end{figure}

Recent work has shown the benefits associated with using action-embeddings \citep{dulac2015deep}, particularly that they allow for generalization over actions.
For real-world problems where there are thousands of possible (discrete) actions, this generalization can significantly speed learning. 
However, this prior work assumes that fixed and predefined representations are provided. 
In this paper we present a method to autonomously learn the underlying structure of the action set by using the observed transitions.
%
% This method can  both improve upon a provided action representation and learn a new action representation from scratch.
%
This method can both learn action representation from scratch and improve upon a provided action representation.

A key component of our proposed method is that it frames the problem of learning an action representation (learning $f$) as a \emph{supervised} learning problem rather than an RL problem. 
This is desirable because supervised learning methods tend to learn more quickly and reliably than RL algorithms since they have access to instructive feedback rather than evaluative feedback \citep{sutton2018reinforcement}. 
The proposed learning procedure exploits the structure in the action set by aligning actions based on the similarity of their impact on the state. 
Therefore, updates to a policy that acts in the space of learned action representation generalizes the feedback received after taking an action to other actions that have similar representations. 
Furthermore, we prove that our combination of supervised learning (for $f$) and reinforcement learning (for $\pi_i$) within one larger RL agent preserves the almost sure convergence guarantees provided by policy gradient algorithms \citep{borkar1997actor}. 

To evaluate our proposed method empirically, we study two real-world recommendation problems using data from widely used commercial applications. 
In both applications, there are thousands of possible recommendations that could be given at each time step (e.g., which video to suggest the user watch next, or which tool to suggest to the user next in multi-media editing software). 
Our experimental results show our proposed system's ability to significantly improve performance relative to existing methods for these applications by quickly and reliably learning action representations that allow for meaningful generalization over the large discrete set of possible actions. 

The rest of this paper is organized to provide in the following order: a background on RL, related work, and the following primary contributions:
    	\begin{itemize}
    	    \item A new parameterization, called the \textit{overall policy}, that leverages action representations. We show that for all optimal policies, $\pi^*$, there exist parameters for this new policy class that are equivalent to $\pi^*$. % Proof of existence of different parameters under this new policy class that can represent any optimal policy.
            \item A proof of equivalence of the policy gradient update between the overall policy and the \textit{internal policy}. 
    	    \item A supervised learning algorithm for learning action representations ($f$ in Figure \ref{Fig:execution}). This procedure can be combined with any existing policy gradient method for learning the overall policy.
    	    \item An almost sure asymptotic convergence proof for the algorithm, which extends existing results for actor-critics \cite{borkar1997actor}.
    	    \item Experimental results on real-world domains with thousands of actions using actual data collected from recommender systems. 
    	\end{itemize}

	   % \textcolor{red}{
	   % \begin{itemize}
	   %     \item Paragraph saying the current gap in knowledge: what cannot be done right now, and why does it matter.
	   %     \item Paragraph listing examples of applications that are held back by this limitation.
	   %     \item Paragraph or two giving more details about one of the applications you show at the end (HelpX?). Include figure.
	   %     \item Paragraph saying what you do in this paper - learn action embeddings.
	   % \end{itemize}
	   % }

	\section{Background}
    	We consider problems modeled as discrete-time \textit{Markov decision processes} (MDPs) with discrete states and finite actions. 
    	An MDP is represented by a tuple, $\mathcal{M} = (\mathcal{S},\mathcal{A},\mathcal{P},\mathcal{R}, \gamma, d_0)$.
    	$\mathcal{S}$ is the set of all possible states, called the state space, and $\mathcal{A}$ is a finite  set of actions, called the action set. 
    	Though our notation assumes that the state set is finite, our primary results extend to MDPs with continuous states.
        In this work, we restrict our focus to MDPs with finite action sets, and $|\mathcal{A}|$ denotes the size of the action set.
    	The random variables, $S_t \in \mathcal S$, $A_t \in \mathcal A$, and $R_t \in \mathbb R$ denote the state, action, and reward at time $t \in \{0,1,\dotsc\}$.
    	We assume that $R_t \in [-R_\text{max},R_\text{max}]$ for some finite $R_\text{max}$. 
    	The first state, $S_0$, comes from an initial distribution, $d_0$, and the reward function $\mathcal R$ is defined so that $\mathcal R(s,a)=\mathbf{E}[R_t|S_t=s,A_t=a]$ for all $s \in \mathcal S$ and $a \in \mathcal A$. 
    	Hereafter, for brevity, we write $P$ to denote both probabilities and probability densities,
    	and when writing probabilities and expectations, write $s,a,s'$ or $e$ to denote both elements of various sets \emph{and} the events $S_t=s, A_t=a$, $S_{t+1}=s'$, or $E_t=e$ (defined later). 
    	The desired meaning for $s,a,s'$ or $e$ should be clear from context. 
    	The reward discounting parameter is given by $\gamma \in [0,1)$. 
    	$\mathcal{P}$ is the state transition function, such that $\forall s,a,s',t,$ $ \mathcal{P}(s,a,s') \coloneqq P(s'| s,a)$.

    	% Phil: Note - you never said here that pi(s,a) = Pr(A_t=a|S_t=s). You need to say this. I'm rewording below. 
    	% Yash: ok
    	A policy $\pi:\mathcal A \times \mathcal S \to [0,1]$ is a conditional distribution over actions for each state: $\pi(a, s)\coloneqq P(A_t=a|S_t=s)$ for all $s \in \mathcal S, a \in \mathcal A$, and $t$.
    	Although $\pi$ is simply a function, we write $\pi(a|s)$ rather than $\pi(a,s)$ to emphasize that it is a conditional distribution. 
    	For a given $\mathcal{M}$, an agent's goal is to find a policy that maximizes the expected sum of discounted future rewards.
    	For any policy $\pi$, the corresponding state-action value function is $q^\pi(s,a) = \mathbf{E}[\sum_{k=0}^{\infty}\gamma^k R_{t+k} |S_t=s, A_t=a, \pi]$, where conditioning on $\pi$ denotes that $A_{t+k} \sim \pi(\cdot|S_{t+k})$ for all $A_{t+k}$ and $S_{t+k}$ for $k \in [t+1, \infty)$.
    	The state value function is $v^\pi(s) = \mathbf{E}[\sum_{k=0}^{\infty}\gamma^k R_{t+k} |S_t=s, \pi]$.  
    	It follows from the Bellman equation that $ v^\pi(s)= \sum_{a \in \mathcal{A}} \pi(a|s) q^\pi(s,a)$.
    	An optimal policy is any $\pi^{*} \in \text{argmax}_{\pi \in \Pi}\mathbf{E} [\sum_{t=0}^{\infty}\gamma^tR_t |\pi]$, where $\Pi$ denotes the set of all possible policies, and $v^*$ is shorthand for $v^{\pi^*}$.

\section{Related Work}

	    %
		%Abstraction of the action space has been previously studied in various forms in the literature.
		%
		Here we summarize the most related work and discuss how they relate to the proposed work. 
		%
		%\textbf{Re-parameterization: } Any similarities with re-parameterization trick? One way to perhaps look at our method is to sample from an underlying low dimensional continuous space instead of sampling from large discrete set.
		%
		
		\textbf{Factorizing Action Space: }         To reduce the  size of large action spaces, \citet{pazis2011generalized} considered representing each action in binary format and learning a value function associated with each bit.
		A similar binary based approach was also used as an ensemble method to learning optimal policies for MDPs with large action sets \cite{sallans2004reinforcement}. % was approaches have been also used in multi-RL setup to reduce the overall complexity of large action space \cite{guestrin2002coordinated}.
		For planning problems, \citet{cui2016online, CuiK18} showed how a gradient based search on a symbolic representation of the state-action value function can be used to address scalability issues.
        More recently, it was shown that better performances can be achieved on Atari 2600 games \cite{bellemare2013arcade} when actions are factored into their primary categories \cite{sharma2017learning}.
        All these methods assumed that a handcrafted binary decomposition of raw actions was provided. 
        To deal with discrete actions that might have an underlying continuous representation, \citet{van2009using} used policy gradients with continuous actions and selected the nearest discrete action. 
        This work was extended by \citet{dulac2015deep} for larger domains, where they performed action representation look up,  similar to our approach.
        However, they assumed that the embeddings for the actions are given, \textit{a priori}. 
        Recent work also showed how action representations can be learned using data from expert demonstrations \cite{naturalGuy}.
        We present a method that can learn action representations with no prior knowledge or further optimize available action representations. 
        If no prior knowledge is available, our method learns these representations from scratch autonomously.
        %
        %As compared to the heuristic rule \cite{dulac2015deep} of taking gradients using the action executed, we prove that the gradient using the embedding sampled matches the true expected gradient of the overall policy. 
		%
        %\textbf{Task-independent representations: } Earlier works for task-independent representations mostly focused on learning state features for value functions like Proto-Value functions (PVFs) \cite{mahadevan2007proto} and Successor Representations (SRs) \cite{dayan1993improving}.
        %
        %PVF was extended for Q-values by \cite{osentoski2007learning} and \cite{barreto2017successor} extended SRs with function approximators to learn features that can be useful for multiple tasks. 
        %
        %Complementary to these approaches, our work shows how to learn action representations.
        %
        
        \textbf{Auxiliary Tasks: }
        Previous works showed empirically that supervised learning with the objective to predict a component of a transition tuple $(s,a,r,s')$ from the others, can be useful as an auxiliary method to learn state representations \cite{jaderberg2016reinforcement,combinedVIncent} or to obtain intrinsic rewards  \cite{shelhamer2016loss,pathak2017curiosity}.
        We show how the overall policy itself can be decomposed using an action representation module learned using a similar loss function. 
        %
        
    %     \textbf{Neuroscience: }
    %     Animals  regularly need to decide among the vast combinations of muscle movements available to reach a desired goal.
	   % %
	   % Research in neuroscience suggests that animals decompose their plans into mid-level abstractions, rather than the exact low-level motor controls needed for each movement \cite{jing2004construction,flash2005motor}.
    %   %
    %     Such abstractions of behavior that form the building blocks for motor control are often called \textit{motor primitives} \cite{todorov2003unsupervised}.
    %     %
    %     Animals' central nervous systems control motor primitives, which in turn control low-level movement of individual muscles.
    %     %
    %     Learning and reward, rather than being tied to specific movements of individual muscles, is often associated with mid-level controls and abstractions \cite{lemay2004modularity,mussa2000motor}.
    %     %
    %     This allows animals to quickly generalize learning over large sets of low-level actions.
    %     %
    %     %As a result, the associations between actions are incorporated into learning; thus, learning and decision-making become more tractable.
    %     %
    %     This paper explores how to automatically impose this structure on machine learning.
    %     %
        
        \textbf{Motor Primitives: } 
	    Research in neuroscience suggests that animals decompose their plans into mid-level abstractions, rather than the exact low-level motor controls needed for each movement \cite{jing2004construction}.
        Such abstractions of behavior that form the building blocks for motor control are often called \textit{motor primitives} \cite{lemay2004modularity,mussa2000motor}.
    	In the field of robotics, dynamical system based models have been used to construct \textit{dynamic movement primitives} (DMPs) for continuous control  \cite{ijspeert2003learning,schaal2006dynamic}.
    	Imitation learning can also be used to learn DMPs, which can be fine-tuned  online using RL \cite{kober2009policy,kober2009learning}.
    	However, these are significantly different from our work as they are specifically parameterized for robotics tasks and produce an encoding for kinematic trajectory plans, not the actions.

    	Later, \citet{thomas2012motor} showed how a goal-conditioned policy can be learned using multiple motor primitives that control only useful sub-spaces of the underlying control problem. 
    	To learn binary motor primitives, \citet{thomas2011conjugate} showed how a policy can be modeled as a composition of multiple ``coagents", each of which learns using only the local policy gradient information \cite{thomas2011policy}.
    	Our work follows a similar direction, but we focus on automatically learning optimal continuous-valued action representations for discrete actions.
    	For action representations, we present a method that uses supervised learning and restricts the usage of high variance policy gradients to train the internal policy only. 
		\textbf{Other Domains: } In supervised learning, representations of the output categories have been used to extract additional correlation information among the labels.
		%
		%Since representations of the classes can also be learned from data available through other sources, they often have an edge over self-contained approaches. 
		%
		Popular examples include learning label embeddings for image classification \cite{akata2016label} and learning word embeddings for natural language problems \cite{mikolov2013distributed}.
		In contrast, for an RL setup, the policy is a function whose outputs correspond to the available actions. 
		We show how learning action representations can be beneficial as well.

    \section{Generalization over Actions}

The benefits of capturing the structure in the underlying state space of MDPs is a well understood and a widely used concept in RL.
State representations allow the policy to generalize across states.
Similarly, there often exists additional structure in the space of actions that can be leveraged.
We hypothesize that exploiting this structure can enable quick generalization across actions, thereby making learning with large action sets feasible. 
To bridge the gap, we introduce an action representation space, $\mathcal{E} \subseteq \mathbb{R}^{ d}$, and consider a factorized policy, $\pi_o$, parameterized by an embedding-to-action mapping function, $f \colon \mathcal{E} \to \mathcal{A}$, and an internal policy, $\pi_i \colon \mathcal{S} \times \mathcal{E} \to [0,1]$, such that the distribution of $A_t$ given $S_t$ is characterized by:% using the following two components:
\begin{equation}
        E_t \sim \pi_i(\cdot |S_t), \hspace{2cm}  A_t = f(E_t). \label{eqn:decomposed-policy}
\end{equation}
Here, $\pi_i$ is used to sample $E_t \in \mathcal{E}$, and the function $f$ deterministically maps this representation to an action in the set $\mathcal{A}$.
Both these components together form an \textit{overall policy}, $\pi_o$.
Figure \ref{Fig:representation} illustrates the probability of each action under such a parameterization.
With a slight abuse of notation, we use $f^{-1}(a)$ as a one-to-many function that denotes the set of representations  that are mapped to the action $a$ by the function $f$, i.e., $f^{-1}(a) \coloneqq \{e\in\mathcal E:f(e)=a\}$. 
\begin{figure}[h]
\centering
\includegraphics[scale=0.22]{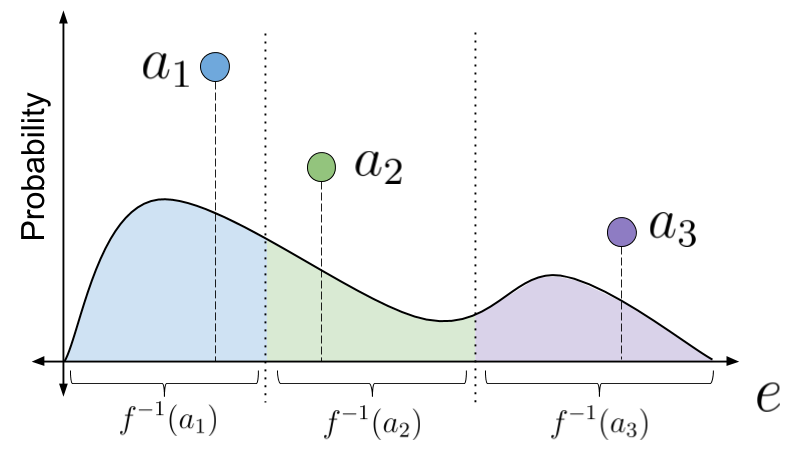}
\caption{
Illustration of the  probability induced for three actions by the probability density of $\pi_i(e,s)$ on a $1$-D embedding space. The $x$-axis represents the embedding, $e$, and the $y$-axis represents the probability. The colored regions represent the mapping $a=f(e)$, where each color is associated with a specific action. }
\label{Fig:representation}
\end{figure}

%The internal policy, $\pi_i$, that selects an action representations and \textbf{b)} a function, $f$, that transforms the selected representation to an actual action.
%
In the following sections we discuss the existence of an optimal policy $\pi_o^*$ and the learning procedure for $\pi_o$.
To elucidate the steps involved, we split it into four parts. 
First, we show that there exists $f$ and $\pi_i$ such that $\pi_o$ is an optimal policy.
Then we present the supervised learning process for the function $f$ when $\pi_i$ is fixed.
Next we give the policy gradient learning process for $\pi_i$ when $f$ is fixed.
Finally, we combine these methods to learn $f$ and $\pi_i$ simultaneously.

\subsection{Existence of $\pi_i$ and $f$ to Represent An Optimal Policy} 
   
% Yash: Adding this line in start to give a topic/motivation for what will follow next.
%With the underlying idea of our proposed overall policy, $\pi_o$, introduced in \eqref{eqn:decomposed-policy}, 
%
In this section, we aim to establish a condition under which $\pi_o$ can represent an optimal policy. 
Consequently, we then define the optimal set of $\pi_o$ and $\pi_i$ using the proposed parameterization.    
To establish the main results we begin with the necessary assumptions.

%        We hypothesize that there exists a structure in the underlying the space of actions that can be leveraged.
%
The characteristics of the actions can be naturally associated with how they influence state transitions.
In order to learn a representation for actions that captures this structure, we consider a standard Markov property, often used for learning probabilistic graphical models \cite{ghahramani2001introduction}, and make the following assumption that the transition information can be sufficiently encoded to infer the action that was executed. 
\begin{assumpA}
\label{ass:A1}
Given an embedding $E_t$, $A_t$ is conditionally independent of $S_t$ and $S_{t+1}$:
{\small
$$%\begin{align}
    %\int_{\mathcal{E}} \!P(A_t=a|E_t=e) %P(E_t=e|S_t=s,S_{t+1}=s')\,\mathrm{d}e,
    P(A_t|S_t,S_{t+1}) =\!\!
    \int_{\mathcal{E}} \!\!\!P(A_t|E_t=e) P(E_t=e|S_t,S_{t+1})\,\mathrm{d}e.
$$%\end{align}
}
% $$  P(a|s, s') = 
%     \int_{\mathcal{E}} \!\!\!P(a|e) P(e|s, s')\,\mathrm{d}e,
% $$%\end{align}
%
\end{assumpA}
%
%where $P(a|s,s')$ and $P(e|s,s')$ are the probability of action, $a$, and the probability density of the continuous latent variable, $e$, respectively, given the consecutive states, $s$ and $s'$.
%
%$P(a|e)$ is the probability of the action, $a$, given $e$. 
%
\begin{assumpA}
\label{ass:A2}
Given the embedding $E_t$ the action, $A_t$ is deterministic and is represented by a function $f:\mathcal E \to \mathcal A$, i.e., $\exists a \text{ such that } P(A_t=a|E_t=e)=1$.
\end{assumpA}   
%    

%
        %
        % \begin{assumpA}
    	   %  There exists a function $f:\mathcal E \to \mathcal A$ such that 
    	   % % deterministic distribution $P(\cdot|e)$ over $\mathcal{A}$, represented by a function, $f$, such that
    	   %%\begin{align}
        % %         P(A_t=a|s,s') &= \int_e \!\! P(A_t=a|e)P(E_t=e|s,s')\,\mathrm{d}e \label{Eqn:action-prob}.
        % %     \end{align}
        %     %\begin{align}
        %         $P(a|s,s') = \int_{f^{-1}(a)} \! P(e|s,s')\,\mathrm{d}e.$ %\label{Eqn:action-prob}. % not referenced - can inline
        %     %\end{align}
        % \end{assumpA} 
    %   Here, $P(a|s,s')$ and $P(e|s,s')$ are the probability of action, $a$, and the probability density of the continuous latent variable, $e$, respectively, given the consecutive states, $s$ and $s'$. 
       %
       %	It represents a function, $f$, that transforms the latent representation, $e$, to a unique action, $a$.
       
        We now establish a necessary condition under which our proposed policy can represent an optimal policy. 
        This condition will also be useful later when deriving learning rules.
    	\begin{lemma}  
    	    \label{lemma:bellman}
    	    Under Assumptions \eqref{ass:A1}--\eqref{ass:A2}, which defines a function $f$, for all $\pi$, there exists a $\pi_i$ such that
            \begin{align}
                v^\pi(s) = \sum_{a \in \mathcal{A}} \int_{f^{-1}(a)} \pi_i(e|s) q^\pi(s, a)\, \mathrm{d}e. \label{eqn:lemma-1}
            \end{align}
        \end{lemma}
         The proof is deferred to the Appendix \ref{Appendix:lemma-1}. 
         Following Lemma \eqref{lemma:bellman}, we use $\pi_i$ and $f$ to define the overall policy as
        \begin{align}
            \pi_o(a|s) &\coloneqq \int_{f^{-1}(a)}\pi_i(e|s)\,\mathrm{d}e.
            \label{eqn:optimal-policy}
        \end{align}

        \begin{thm} Under Assumptions \eqref{ass:A1}--\eqref{ass:A2}, which defines a function $f$,
        there exists an overall policy, $\pi_o$, such that $v^{\pi_o}=v^{\star}$.   \label{thm:optimal-overall-policy}
        \end{thm}

        \begin{proof}
        This follows directly from Lemma \ref{lemma:bellman}.
        Because the state and action sets are finite, the rewards are bounded, and $\gamma \in [0,1)$, there exists at least one optimal policy. 
        For any optimal policy $\pi^\star$, the corresponding state-value and  state-action-value functions are the unique $v^\star$ and $q^\star$, respectively.
       By Lemma \ref{lemma:bellman} there exist $f$ and $\pi_i$ such that 
       \begin{align}
           v^\star(s)&=  \sum_{a \in \mathcal{A}}  \int_{f^{-1}(a)}\pi_i(e|s) q^\star(s,a)\,\mathrm{d}e. \label{eqn:optimal-overall-policy}
       \end{align}
   Therefore, there exists $\pi_i$ and $f$, such that the resulting $\pi_o$ has the state-value function $v^{\pi_o}=v^{\star}$, and hence it represents an optimal policy.
        \end{proof}
        Note that Theorem \ref{thm:optimal-overall-policy} establishes existence of an optimal overall policy based on equivalence of the state-value function, but does \emph{not} ensure that all optimal policies can be represented by an overall policy. 
        Using \eqref{eqn:optimal-overall-policy}, we define $\Pi_o^\star \coloneqq \{\pi_o : v^{\pi_o}=v^\star\}$.
        Correspondingly, we define the set of \textit{optimal internal policies} as $\Pi_i^\star \coloneqq \{\pi_i : \exists \pi_o^\star \in \Pi_o^\star,\exists f, \pi_o^\star(a|s) = \int_{f^{-1}(a)}\pi_i(e|s)\,\mathrm{d}e \}$. % for any $\pi_o^\star \in \Pi_o^\star$.  %   %
        %
       %
       %As there can be multiple $\pi^\star \in \Pi$ corresponding to $v^\star$, all of them might not be representable using $\pi_o$.
       %
       %However. at least one optimal policy, $\pi^\star$, is. 
        %
        \subsection{Supervised Learning of $f$ For a Fixed $\pi_i$}    
        \label{section:learn-f}    
        %
        % Phil: I'm not sure why we're saying the line below - what point is this supporting? I'm cutting for now.
        % Yash: Agreed. I was trying to highlight the fact that the f as introduced in assumption 1 had nothing to do with (optimal) policy, however, through Thm 1, we can indeed use it as a part of parameterization. I have moved this to next section.
        
        Theorem \ref{thm:optimal-overall-policy} shows that there exist $\pi_i$ and a function $f$, which helps in predicting the action responsible for the transition from $S_t$ to $S_{t+1}$, such  that the corresponding overall policy is optimal.  
        %
        %Theorem \ref{thm:optimal-overall-policy} showed that there exist $f$ and $\pi_i$ such that the corresponding overall policy is optimal. %, we obtained a required characteristic of the function $f$ needed to obtain an optimal overall policy, $\pi_o^\star$.
        %
        However, such a function, $f$, may not be known \emph{a priori}.
        In this section, we present a method to estimate $f$ using data collected from interactions with the environment.

        By Assumptions \eqref{ass:A1}--\eqref{ass:A2}, $P(A_t|S_t,S_{t+1})$ can be written in terms of $f$ and $P(E_t|S_t,S_{t+1})$. 
        We propose searching for an estimator, $\hat f$, of $f$ and an estimator, $\hat g(E_t|S_t,S_{t+1})$, of $P(E_t|S_t,S_{t+1})$ such that a reconstruction of $P(A_t|S_t,S_{t+1})$ is accurate. 
        Let this estimate of $P(A_t|S_t,S_{t+1})$ based on $\hat f$ and $\hat g$ be
        {
        \begin{equation}
            \small
            \hat P(A_t|S_t,S_{t+1})  =  \int_\mathcal{E} \!\! \hat f (A_t|E_t\!=\!e) \hat g(E_t\!=\!e|S_t,S_{t+1})\,\mbox{d}e \label{eqn:action-rep-estimator}
        \end{equation}\textnormal{} 
        }
        %
        %To ensure that the $\hat P(a|s,s')$ obtained using estimators $\hat f$ and $\hat g$, is a close approximate to the true probability, $P(a|s,s')$, we formulate an objective to minimize the error in the approximation. 
        %
        One way to measure the difference between $P(A_t|S_t,S_{t+1})$ and $\hat P(A_t|S_t,S_{t+1})$ is using the expected (over states coming from the on-policy distribution) Kullback-Leibler (KL) divergence
        \begin{align}
             =& -\mathbf{E} \left [\sum_{a \in \mathcal{A}} P(a|S_t,S_{t+1}) \ln \left ( \frac{\hat P(a|S_t,S_{t+1})}{P(a|S_t,S_{t+1}) } \right ) \right ]%\label{eqn:KL}
            \\
            =& -\mathbf{E} \left [ \ln \left ( \frac{\hat P(A_t|S_t,S_{t+1})}{P(A_t|S_t,S_{t+1})} \right )  \right ]. \label{eqn:KL-sample}
        \end{align}
        %
        % \textcolor{red}{[Not sure how to strongly justify either on-policy or off-policy.] Notice that the policy used to select actions can impact the distribution of $A_t$ given $S_t$ and $S_{t+1}$, and so we must condition on the policy, $\pi$, above. 
        % %
        % For now, we ignore this dependence on $\pi$ and discuss the impact of this choice later. 
        % }
        
        Since the observed transition tuples, $(S_t,A_t,S_{t+1})$, contain the action responsible for the given $S_t$ to $S_{t+1}$ transition,
        %, it has information about $P(A_t|S_t,S_{t+1})$.
        %
        an on-policy sample estimate of the KL-divergence 
        %\eqref{eqn:KL}
        can be computed readily using \eqref{eqn:KL-sample}.
        We adopt the following loss function based on the KL divergence between $P(A_t|S_t,S_{t+1})$ and $\hat P(A_t|S_t,S_{t+1})$:
        \begin{align}
            \mathcal{L}(\hat f, \hat g) &= - \mathbf{E}\left [ \ln \left (\hat P(A_t|S_t,S_{t+1}) \right )\right ], 
            \label{Eqn:self-supervised-loss}
        \end{align}
        where the denominator in \eqref{eqn:KL-sample} is not included in \eqref{Eqn:self-supervised-loss} because it does not depend on $\hat f$ or $\hat g$. 
        If $\hat f$ and $\hat g$ are parameterized, their parameters can be learned by minimizing the loss function, $\mathcal{L}$, using a supervised learning procedure.
           	\begin{figure}[t]
    		\centering
    		\includegraphics[scale=0.2]{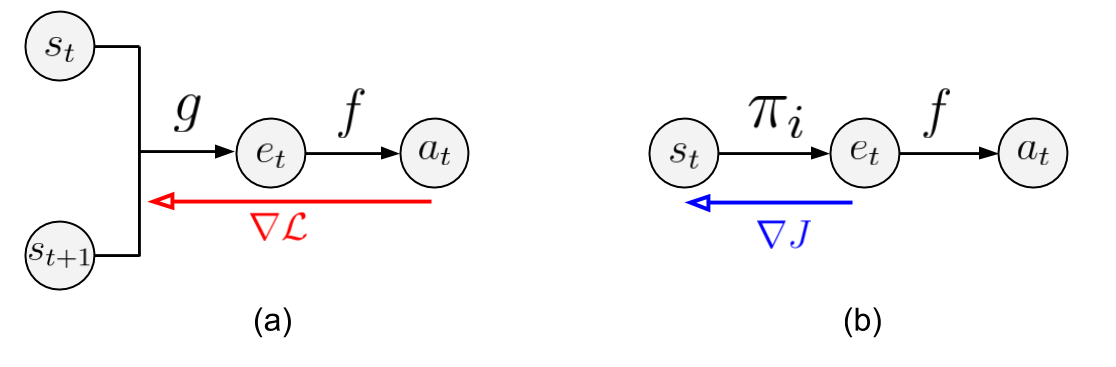}
    		\caption{%
    		%\textcolor{red}{In the left plot, $f$ maps $a_t$ to $e_t$, not $e_t$ to $a_t$. So, above the line should be $f^{-1}$.} 
    		(a) Given a state transition tuple, functions $g$ and $f$ are used to estimate the action taken. 
    		The red arrow denotes the gradients of the supervised loss \eqref{Eqn:self-supervised-loss} for learning the parameters of these functions.
    		%functions and satisfying Assumption \ref{ass:A1}.
    		%
    		(b) During execution, an internal policy, $\pi_i$, can be used to first select an action representation, $e$. 
    		The function $f$, obtained from previous learning procedure, then transforms this representation to an action.
    		The blue arrow represents the internal policy gradients \eqref{Eqn:internal_gradient} obtained using Lemma \ref{prop:local-policy-gradient} to update $\pi_i$.
    		}
    		\label{Fig:architecture-graph}
    	\end{figure}

        A computational graph for this model is shown in Figure \ref{Fig:architecture-graph}. 
        We refer the reader to the Appendix \ref{Appendix:parameter} for the parameterizations of $\hat f$ and $\hat g$ used in our experiments.
        Note that, while $\hat f$ will be used for $f$ in an overall policy, $\hat g$ is only used to find $\hat f$, and will not serve an additional purpose.
%         One way to measure the distance between $P(a|s,s')$ and $\bar P(a|s,s')$ is by using the  Kullback-Leibler (KL) divergence:
%         \begin{align}
%             D_{KL}(P||\bar P) &= -\sum_{a \in \mathcal{A}} P(a|s,s') \log \frac{\bar P(a|s,s')}{P(a|s,s') }
%         \end{align}
%         %
%         Since the observed trajectories itself contains the transitions from the distribution $P(a|s,s')$, we can readily compute the loss using stochastic estimate of $D_{KL}(P||\bar P)$:
%         \begin{align}
%             \mathcal{L} &= - \mathbf{E}\lbrack\log \bar P(a|s,s') |s,s'\rbrack \label{Eqn:self-supervised-loss}
%         \end{align}
%         where the expectation is over all the actions and the second term of KL is independent of the parameters being optimized and is thus ignored. 
%         %
%         The loss in \eqref{Eqn:self-supervised-loss} can be easily minimized using any gradient based update rule to learn the functions $f$ and $g$. 

        As this supervised learning process only requires estimating $P(A_t|S_t,S_{t+1})$, 
        it does not require (or depend on) the rewards. 
        This partially mitigates the problems due to sparse and stochastic rewards, since an alternative informative supervised signal is always available.
        This is advantageous for making the action representation component of the overall policy learn quickly and with low variance updates.

       \subsection{Learning $\pi_i$ For a Fixed $f$}
        \label{section:learn-internal-policy}
    	A common method for learning a policy parameterized with weights $\theta$ is to optimize the discounted start-state objective function, 
    	$ %\begin{align}
    	    J(\theta) := \sum_{s \in \mathcal{S}} d_0(s) v^\pi(s). %\label{Eqn:Performance-function},
    	$ %\end{align}
    	For a policy with weights $\theta$, the expected performance of the policy can be improved by ascending the \emph{policy gradient}, $\frac{\partial J(\theta)}{\partial \theta}$. 

    	Let the state-value function associated with the internal policy, $\pi_i$, be $v^{\pi_i}(s) =  \mathbf{E}[\sum_{t=0}^{\infty}\gamma^t R_{t} |s, \pi_i, f]$, and the state-action value function $q^{\pi_i}(s,e) = \mathbf{E}[\sum_{t=0}^{\infty}\gamma^t R_{t} |s, e, \pi_i, f]$.
    	 We then define the performance function for $\pi_i$ as:
        \begin{align}
    	    J_i(\theta) := \sum_{s \in \mathcal{S}} d_0(s) v^{\pi_i}(s) \label{Eqn:internal-Performance-function}.
    	\end{align}
    	Viewing the embeddings as the action for the agent with policy $\pi_i$, the policy gradient theorem \cite{sutton2000policy}, states that the unbiased \cite{thomas2014bias} gradient of \eqref{Eqn:internal-Performance-function} is,
        \begin{align}
        	%\nabla J_i(\theta) = \mathbf{E}_s\left [ \int_e \!\! \pi_i(e|s) \frac{\partial}{\partial \theta} \ln (\pi_i(e|s)) \cdot q^\pi(s, e) \mathrm{d}e\right ],
        	\frac{\partial J_i(\theta)}{\partial \theta} = \sum_{t=0}^{\infty}\mathbf{E}\left [  \gamma^t \int_\mathcal{E} q^{\pi_i}(S_t, e) \frac{\partial}{\partial \theta} \pi_i(e|S_t)  \, \mathrm{d}e\right ],
    	    \label{Eqn:internal_gradient}
    	\end{align}
    	where, the expectation is over states from $d^\pi$, as defined by \citet{sutton2000policy} (which is not a true distribution, since it is not normalized). 
    	The parameters of the internal policy can be learned by iteratively updating its parameters in the direction of $\partial J_i(\theta) /\partial \theta$.
    	Since there are no special constraints on the policy $\pi_i$, any policy gradient algorithm designed for continuous control, like DPG \cite{silver2014deterministic}, PPO \cite{schulman2017proximal}, NAC \cite{bhatnagar2009natural} etc., can be used out-of-the-box.

    	However, note that the performance function associated with the overall policy, $\pi_o$ (consisting of function $f$ and the internal policy parameterized with weights $\theta$), is:
    	\begin{align}
    	    J_o(\theta,f) = \sum_{s \in \mathcal{S}} d_0(s) v^{\pi_o}(s) \label{Eqn:Performance-function}.
    	\end{align}
    	The ultimate requirement is the improvement of this overall performance function, $J_o(\theta,f)$, and not just $J_i(\theta)$.
        So, how useful is it to update the internal policy, $\pi_i$, by following the gradient of its own performance function? The following lemma answers this question. 
    
        \begin{lemma}
         For all deterministic functions, $f$, which map each point, $e \in \mathbb{R}^{ d}$, in the representation space to an action, $a \in \mathcal{A}$, the expected updates to $\theta$ based on $\frac{\partial J_i(\theta)}{\partial \theta}$ are equivalent to updates based on $\frac{\partial J_o(\theta,f)}{\partial \theta}$. 
    	That is,
    	\begin{align*}
        	 \frac{\partial J_o(\theta,f)}{\partial \theta} = \frac{\partial J_i(\theta)}{\partial \theta}.
    	\end{align*}
    	\label{prop:local-policy-gradient}
    	\end{lemma}
    	The proof is deferred to the Appendix \ref{Appendix:lemma-2}.
        The chosen parameterization for the policy has this special property, which allows $\pi_i$ to be learned using its internal policy gradient.
       	Since this gradient update does not require computing the value of any $\pi_o(a|s)$ explicitly,  
       	%depend on the function $f$, \textcolor{red}{[Yes it does - $v^{\pi_i}(s)$ depends on $f$...]} 
       	% Yash: Yes, my bad. It was loosely worded.
       	the potentially intractable computation of $f^{-1}$ in \eqref{eqn:optimal-policy} required for $\pi_o$ can be avoided.
       	Instead, $\partial J_i(\theta) / \partial \theta$ can be used directly to update the parameters of the internal policy while still optimizing the overall policy's performance, $J_o(\theta,f)$.
           
        \subsection{Learning $\pi_i$ and $f$ Simultaneously} 
        \label{section:learn-simultaneously}
        Since the supervised learning procedure for $f$ does not require rewards, a few initial trajectories can contain enough information to begin learning a useful action representation. 
        As more data becomes available it can be used for fine-tuning and improving the action representations.
        % However, there are two primary reasons to keep updating the representations online:
        % \begin{itemize}
        %     \item With more interactions, more data becomes available, which can be used for fine-tuning and improving the action representations.
        %     \item In cases where multiple actions can result in the same state transitions, transition tuples from the on-policy distribution are required such that actions chosen by the current policy can be used to learn $f$.
        % \end{itemize}
        %
	\subsubsection{Algorithm}
		We call our algorithm \textbf{p}olicy \textbf{g}radients with \textbf{r}epresentations for \textbf{a}ctions (PG-RA). 
		%
		%The pictorial depiction of the algorithm is presented in Figure \ref{Fig:architecture}. 
		%
		%
		PG-RA first initializes the parameters in the action representation component by sampling a few trajectories using a random policy and using the supervised loss defined in \eqref{Eqn:self-supervised-loss}.
		If additional information is known about the actions, as assumed in prior work \cite{dulac2015deep}, it can also be considered when initializing the action representations. 
		Optionally, once these action representations are initialized, they can be kept fixed.

		In the Algorithm \ref{Alg:1}, Lines $2$-$9$ illustrate the online update procedure for all of the parameters involved. 
		Each time step in the episode is represented by $t$.
		For each step, an action representation is sampled and is then mapped to an action by $\hat f$. 
		Having executed this action in the environment, the observed reward is then used to update the internal policy, $\pi_i$, using \textit{any} policy gradient algorithm.
		Depending on the policy gradient algorithm, if a critic is used then semi-gradients of the TD-error are used to update the parameters of the critic.
		In other cases, like in REINFORCE \cite{williams1992simple} where there is no critic, this step can be ignored.
		The observed transition is then used in Line $9$ to update the parameters of $\hat f$ and $\hat g$ so as to minimize the supervised learning loss \eqref{Eqn:self-supervised-loss}. 
		In our experiments, Line $9$ uses a stochastic gradient update. %(e.g., using a stochastic gradient update).

	\IncMargin{1em}
	\begin{algorithm2e}[t]
		% \SetAlgoLined
		%\textbf{Input:} {self\_supervise frequency $k$} \\
		Initialize action representations \\%memory buffer $\Omega$ \\
		\For {$episode = 0,1,2...$}{
			%\If{episode $\%$ $k==0$}{
			%	\textit{Self\_supervise($\Omega$)}
			%}
			\For {$t = 0,1,2...$} {
			    Sample action embedding, $E_t$, from $\pi_i(\cdot|S_t) $ \\
			    $A_t = \hat f(E_t)$\\
			    Execute $A_t$ and observe $S_{t+1}, R_{t}$ \\
			    Update $\pi_i$ using \textit{any} policy gradient algorithm\\ 
	            Update critic (if any) to minimize TD error\\
				%      \Comment{\eqref{Eqn:Estimated-dist}} \\
	            Update $\hat f$ and $\hat g$ to minimize $\mathcal L$ defined in \eqref{Eqn:self-supervised-loss} 

			}
		}
% 			}
% 		}     
		\caption{Policy Gradient with Representations for Action (PG-RA)}
		\label{Alg:1}  
	\end{algorithm2e}
	\DecMargin{1em}    	
 
    \subsubsection{PG-RA Convergence}
     If the action representations are held fixed while learning the internal policy, then as a consequence of Property \ref{prop:local-policy-gradient}, convergence of our algorithm directly follows from previous two-timescale results \cite{borkar1997actor,bhatnagar2009natural}.
 	Here we show that learning both $\pi_i$ and $f$ simultaneously using our PG-RA algorithm can also be shown to converge by using a three-timescale analysis.

    Similar to prior work \cite{bhatnagar2009natural,degris2012off,konda2000actor}, for analysis of the updates to the parameters, $\theta \in \mathbb{R}^{d_\theta}$, of the internal policy, $\pi_i$, we use a projection operator $\Gamma : \mathbb{R}^{d_\theta} \rightarrow \mathbb{R}^{d_\theta}$ that projects any $x \in \mathbb{R}^{d_\theta}$ to a compact set $\mathcal{C}\subset \mathbb R^{d_\theta}$.
 	We then define an associated  vector field operator, $\hat \Gamma$,
    that projects any gradients leading outside the compact region,  $\mathcal{C}$, back to $\mathcal{C}$.
    We refer the reader to the Appendix \ref{Appendix:PG-RA-convergence} for precise definitions of these operators and the additional standard assumptions (\ref{ass:differentiable})--(\ref{ass:param-bounded}).
 	Practically, however, we do not project the iterates to a constraint region as they are seen to remain bounded (without projection).

   \begin{thm}
    	\label{thm:convergence}
    	  Under Assumptions \eqref{ass:A1}--\eqref{ass:param-bounded}, the internal policy parameters   $\theta_t$, converge to $\mathcal{\hat Z} = \left\{x \in \mathcal{C}|\hat \Gamma\left(\frac{\partial J_i(x)}{\partial \theta}\right)=0\right \}$ as $t \rightarrow \infty$, with probability one.
	    \end{thm}
%      
% 	Let the parameters of the critic and the internal policy be denoted as $\omega$ and $\theta$ respectively. Also, let $\phi$ denote all the parameters of $f$ and $g$. Let $\theta^\star$ be a locally optimal set of parameters for achieving the maximum expected discounted return is defined to be $\omega^\star, \phi^\star$ and $\theta^\star$.
% 	%
  
%     	\begin{thm}
%     	\label{thm:convergence}
%         Under Assumptions 1-6, PG-RA parameters:  $(\omega_t, \phi_t, \theta_t) \rightarrow (\omega^\star,\phi^\star,\theta^\star)$ as $t \rightarrow \infty$, with probability one.
% 	    \end{thm}
%	    
%
	    \begin{proof} (Outline) We consider three learning rate sequences, such that the update recursion for the internal policy is on the slowest timescale, the critic's update recursion is on the fastest, and the action representation module's has an intermediate rate. 
	    With this construction, we leverage the three-timescale analysis technique \cite{borkar2009stochastic} and prove convergence.
	    The complete proof is in the Appendix \ref{Appendix:three-timescale}.
	    \end{proof}

	\section{Empirical Analysis}

  	\begin{figure*}[!ht]
    		\centering
    		\includegraphics[ height=3.2cm, width=4.25cm]{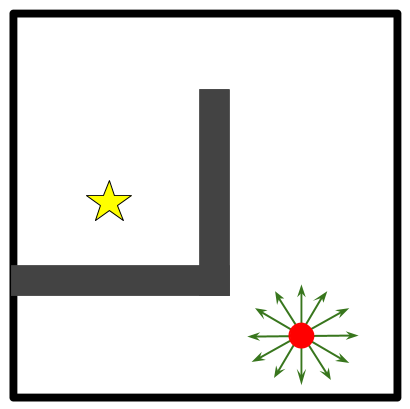}
    		\hspace{10pt}
    		\includegraphics[ height=3.5cm, width=4.5cm]{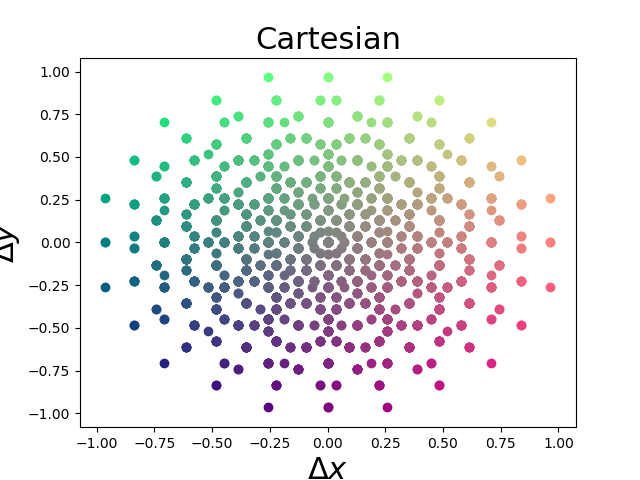}
    		\includegraphics[ height=3.5cm, width=4.5cm]{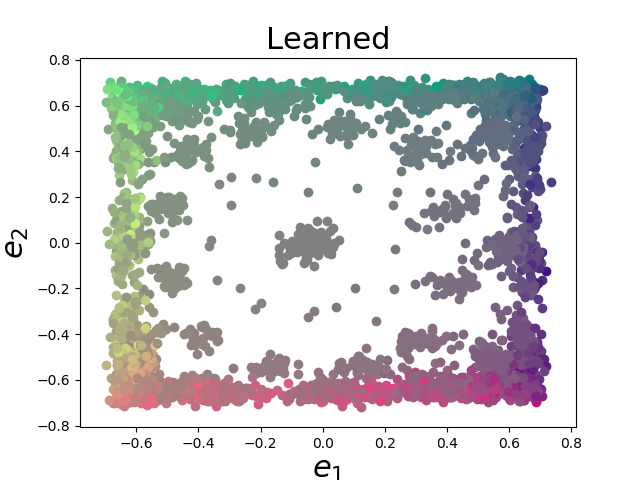}
    		\caption{ (a) The maze environment. 
    		The star denotes the goal state, the red dot corresponds to the agent and the arrows around it are the $12$  actuators.
    		Each action corresponds to a unique combination of these actuators.
    		Therefore, in total $2^{12}$ actions are possible.  
    		(b)  2-D representations for the displacements in the Cartesian co-ordinates caused  by each action, and (c) learned action embeddings.
    		In both (b) and (c), each action is colored based on the displacement ($\Delta x$, $\Delta y$) it produces. 
    		That is, with the color \lbrack R= $\Delta x$, G=$\Delta y$, B=$0.5$\rbrack, where $\Delta x$ and $\Delta y$ are normalized to $[0,1]$ before coloring. 
    		Cartesian actions are plotted on co-ordinates ($\Delta x$, $\Delta y$), and learned ones are on the coordinates in the embedding space.
    	    Smoother color transition of the learned representation is better as it corresponds to preservation of the \textit{relative} underlying structure.
    		The `squashing' of the learned embeddings is an artifact of a non-linearity applied to bound its range.
    		}
    		\label{Fig:emb}
    	\end{figure*}

    	A core motivation of this work is to provide an algorithm that can be used as a drop-in extension for improving the action generalization capabilities of existing policy gradient methods for problems with large action spaces.
    	We consider two standard policy gradient methods: actor-critic (AC) and deterministic-policy-gradient (DPG) \cite{silver2014deterministic} in our experiments.
    	Just like previous algorithms, we also ignore the $\gamma^t$ terms and perform the biased policy gradient update to be practically more sample efficient \cite{thomas2014bias}.
    	We believe that the reported results can be further improved by using the proposed method with other policy gradient methods; we leave this for future work.
    	For detailed discussion on parameterization of the function approximators and hyper-parameter search, see Appendix \ref{Appendix:parameter}.
    	\begin{figure*}[t]
		\centering
		\includegraphics[width=0.29\textwidth]{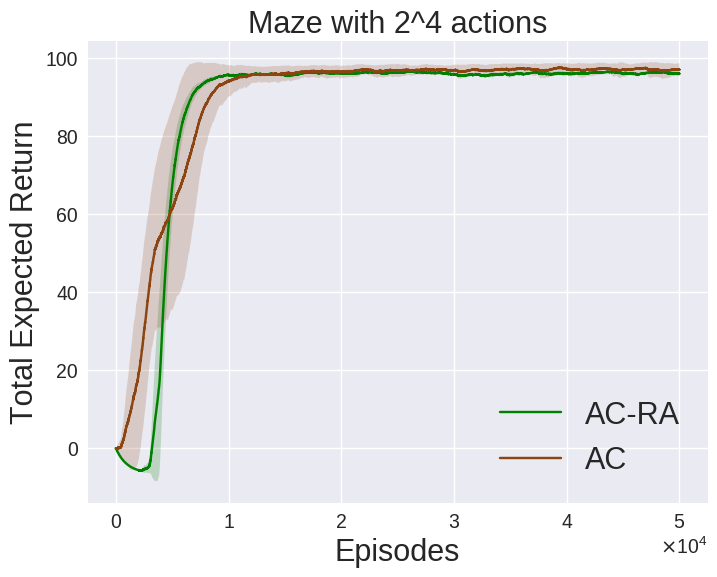} \hfill
		\includegraphics[width=0.29\textwidth]{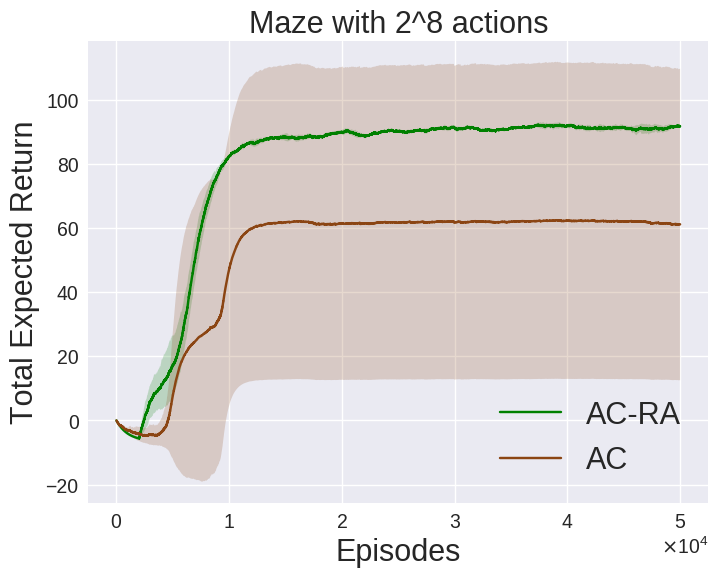} \hfill
		\includegraphics[width=0.29\textwidth]{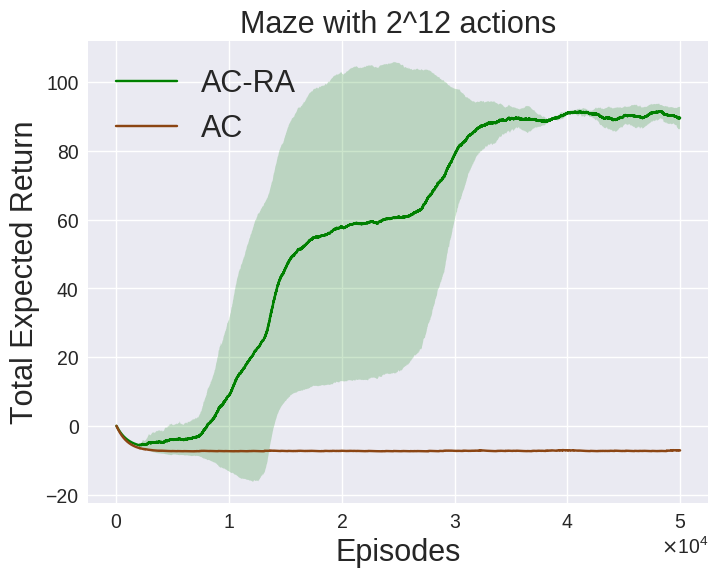} \\
		\includegraphics[width=0.29\textwidth]{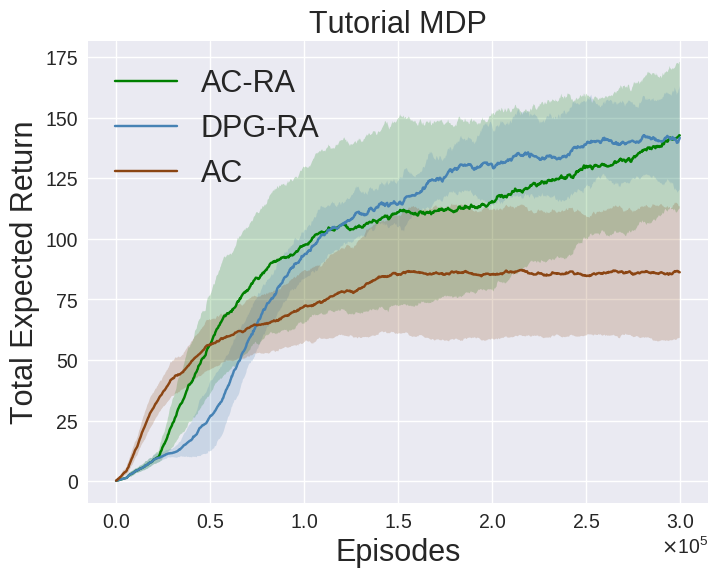}
		\hspace{20pt}
		\includegraphics[width=0.29\textwidth]{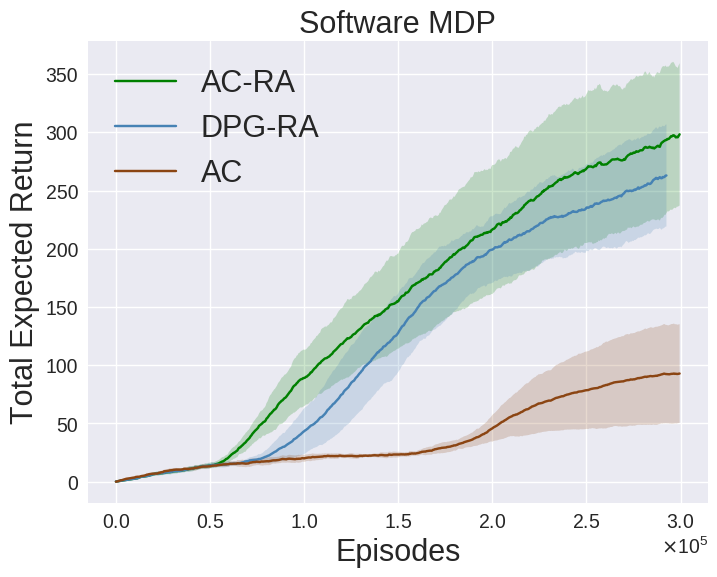}
		\caption{(Top) Results on the Maze domain with $2^4, 2^8,$ and $2^{12}$ actions respectively. (Bottom) Results on a) Tutorial MDP b) Software MDP. AC-RA and DPG-RA are the variants of PG-RA algorithm that uses actor-critic (AC) and DPG, respectively. The shaded regions correspond to one standard deviation and were obtained using $10$ trials.}
		\label{Fig:performance-plots}
	\end{figure*}
\subsection{Domains}
        \paragraph{Maze: } As a proof-of-concept, we constructed a continuous-state maze environment where the state comprised of the coordinates of the agent's current location. 
        The agent has $n$ equally spaced actuators (each actuator moves the agent in the direction the actuator is pointing towards) around it, and it can choose whether each actuator should be on or off. 
        Therefore, the size of the action set is exponential in the number of actuators, that is $|\mathcal{A}| = 2^n$. 
        The net outcome of an action is the vectorial summation of the displacements associated with the selected actuators.
        The agent is rewarded with a small penalty for each time step, and a reward of $100$ is given upon reaching the goal position. 
        To make the problem more challenging, random noise was added to the action $10\%$ of the time and the maximum episode length was $150$ steps. 
        %
        %Though seemingly trivial, the domain is quite challenging because of large action space, long horizon and single goal reward.
    	%
    	
    	This environment is a useful test bed as it requires solving a long horizon task in an MDP with a large action set and a single goal reward.
    	Further, we know the Cartesian representation for each of the actions, and can thereby use it to visualize the learned representation, as shown in Figure \ref{Fig:emb}.
    	%
    	%Figure \ref{Fig:emb} shows that the proposed method can learn these relative associations among the actions. 
    	%

    	\paragraph{Real-word recommender systems: } 
    	We consider two real-world applications of recommender systems that require decision making over \textit{multiple time steps}.

        First, a web-based video-tutorial platform, which has a recommendation engine that suggests a series of tutorial videos on various software.
    	    The aim is to meaningfully engage the users in learning how to use these software and convert novice users into experts in their respective areas of interest.
	    The tutorial suggestion at each time step is made from a large pool of available tutorial videos on several software.

        The second application is a professional multi-media editing software. %, wherein the recommendation engine suggests sequences of tools to use. 
	    Modern multimedia editing software often contain many tools that can be used to manipulate the media, and this wealth of options can be overwhelming for users. 
	    In this domain, an agent suggests which of the available tools the user may want to use next.
	    The objective is to increase user productivity and assist in achieving their end goal. 

        For both of these applications, an existing log of user's click stream data was used to create an n-gram based MDP model for user behavior \cite{shani2005mdp}.
        In the tutorial recommendation task, user activity for a three month period was observed. 
        Sequences of user interaction were aggregated to obtain over $29$ million clicks.   
        Similarly, for a month long duration, sequential usage patterns of the tools in the multi-media editing software were collected to obtain a total of over $1.75$ billion user clicks.  
        Tutorials and tools that had less than $100$ clicks in total were discarded. 
        The remaining $1498$ tutorials and $1843$ tools for the web-based tutorial platform and the multi-media software, respectively, were used to create the action set for the MDP model.
        The MDP had continuous state-space, where each state consisted of the feature descriptors associated with each item (tutorial or tool) in the current n-gram. 
        Rewards were chosen based on a surrogate measure for difficulty level of tutorials and popularity of final outcomes of user interactions in the multi-media editing software, respectively. 
        Since such data is sparse, only $5\%$ of the items had rewards associated with them, and the maximum reward for any item was $100$.

        Often the problem of recommendation is formulated as a contextual bandit or collaborative filtering problem, but as shown by \citet{theocharous2015ad} these approaches fail to capture the long term value of the prediction.
        Solving this problem for a longer time horizon with a large number of  actions (tutorials/tools) makes this real-life problem a useful and a challenging domain for RL algorithms.
\subsection{Results}
\subsubsection*{Visualizing the learned action representations }
To understand the internal working of our proposed algorithm, we present visualizations of the learned action representations on the maze domain.
%
%The maze domain acts as a didactic environment for this purpose.
%
A pictorial illustration of the environment is provided in Figure \ref{Fig:emb}. 
Here, the underlying structure in the set of actions is related to the displacements in the Cartesian coordinates.
This provides an intuitive base case against which we can compare our results.

In Figure \ref{Fig:emb}, we provide a comparison between the action representations learned using our algorithm and the underlying Cartesian representation of the actions.
It can be seen that the proposed method extracts useful structure in the action space.
Actions which correspond to settings where the actuators on the opposite side of the agent are selected result in relatively small displacements to the agent.
These are the ones in the center of plot.
In contrast, maximum displacement in any direction is caused by only selecting actuators facing in that particular direction.
Actions corresponding to those are at the edge of the representation space.
The smooth color transition indicates that not only the information about magnitude of displacement but the direction of displacement is also represented.
Therefore, the learned representations efficiently preserve the relative transition information among all the actions. %
To make exploration step tractable in the internal policy, $\pi_i$, we bound the  representation space along each dimension to the range [$-1,1$] using \textit{Tanh} non-linearity.
This results in `squashing' of these representations around the edge of this range.

\subsubsection*{Performance Improvement }
    % 	The plots in Figure \ref{Fig:performance-plots} show how the performance of vanilla AC method deteriorates drastically as the number of actions increase, even though the goal remains the same. 
    % 	%
    % 	In comparison, with simple addition of our action representation module, it can now efficiently generalize by capturing the underlying structure in action space and solve all the tasks. 
    % 	%
    %  	The plots in Figure \ref{Fig:performance-plots} demonstrate the effectiveness of our method.
    %  	%
    %  	Standard AC methods fail to reason over longer time horizons with so many actions and their
    %  	learning gets stuck, choosing mostly one-step actions that have high return.
    %  	%
    %  	In comparison, just by adding the module to improve generalization by using action representations, the algorithm is now able to perform nearly $2\times$ and $4\times$ in the respective tasks.
    %  	%
    %  	Similar gains can be observed while using DPG with action representations as well.
 
 %

The plots in Figure \ref{Fig:performance-plots} for the Maze domain show how the performance of standard actor-critic (AC) method deteriorates as the number of actions increases, even though the goal remains the same. 	
However, with the addition of an action representation module it is able to capture the underlying structure in the action space and consistently perform well across all settings.
Similarly, for both the tutorial and the software MDPs, standard AC methods fail to reason over longer time horizons under such an overwhelming number of actions, choosing mostly one-step actions that have high returns.  
In comparison, instances of our proposed algorithm are not only able to achieve significantly higher return, up to $2\times$ and $3\times$ in the respective tasks, but they do so much quicker.
These results reinforce our claim that learning action representations allow implicit generalization of feedback to other actions embedded in proximity to executed action.

Further, under the PG-RA algorithm, only a fraction of total parameters, the ones in the internal policy, are learned using the high variance policy gradient updates.
    	The other set of parameters associated with action representations are learned by a supervised learning procedure.
    	This reduces the variance of updates significantly, thereby making the PG-RA algorithms learn a better policy faster.
        This is evident from the plots in the Figure  \ref{Fig:performance-plots}.
These advantages allow the internal policy, $\pi_i$, to quickly approximate an optimal policy without succumbing to the curse of large actions sets. 

	\section{Conclusion}
% 	State representations help in taking `similar' actions for states with `similar' representation. 
% %
% Once an action is found useful to be executed in a state, state representations help in quickly generalizing it across states.
% %
% Though useful, this does not directly help in finding the useful action that needs to be taken in the first place. 
% %
% The impact of this is pronounced in domains that have large action sets or a set of actions where new actions get added over time.
% %
% Using the standard approach, it requires significantly larger number of samples to learn about every single action's consequence in different states as it has no means to quickly generalize the observed feedback across other actions. 
% %
% Action representations help to bridge this gap.
% %

	    %
        In this paper, we built upon the core idea of leveraging the structure in the space of actions and showed its importance for enhancing generalization over large action sets in real-world large-scale applications. 
        Our approach has three key advantages. (a) Simplicity: by simply using the observed transitions, an additional supervised update rule can be used to learn action representations. (b) Theory: we showed that the proposed overall policy class can represent an optimal policy and derived the associated learning procedures for its parameters.  (c) Extensibility: as the PG-RA algorithm indicates, our approach can be easily extended using other policy gradient methods to leverage additional advantages, while preserving the convergence guarantees. 
        
        %
        % We formally showed how a new class of policies can be obtained as composition of a function associated with action representations and an internal policy which acts in this representation space.
        % %
        % This formulation leveraged a reward independent supervised learning rule for action representations and an internal policy update rule using any existing policy gradient method.
        % %

        % An interesting future direction would be to extend the results for capturing the structure of a high dimensional continuous action space ($\in \mathbb{R}^m$) into a lower dimensional representation space ($\in \mathbb{R}^n$) as well.
        %
        % Unlike finite set of actions that can be embedded in a continuous space,
        %
        % the key challenge is that learning lower dimensional space for continuous action inevitably results in the inability to represent some sections of the original action space (as $n<m$).

\section*{Acknowledgement}
The research was supported by and partially conducted at Adobe Research.
We thank our colleagues from the Autonomous Learning Lab (ALL) for the valuable discussions that greatly assisted the research.
We also thank Sridhar Mahadevan for his valuable comments and feedback on the earlier versions of the manuscript.
We are also immensely grateful to the four anonymous reviewers who shared their insights to improve the paper and make the results stronger.
\bibliography{example_paper}
\bibliographystyle{icml2019}

\newpage
\onecolumn
\appendix
\section*{Appendix}

\setcounter{lemma}{0}
\setcounter{thm}{1}

\section{Proof of Lemma 1}
\label{Appendix:lemma-1}
   	\begin{lemma}  
    	   % \label{lemma:bellman}
    	    Under Assumptions \eqref{ass:A1}--\eqref{ass:A2}, which defines a function $f$, for all $\pi$, there exists a $\pi_i$ such that
            \begin{align}
                v^\pi(s) = \sum_{a \in \mathcal{A}} \int_{f^{-1}(a)} \pi_i(e|s) q^\pi(s, a)\, \mathrm{d}e.
            \end{align}
        \end{lemma}

        \begin{proof}
    
    	The Bellman equation associated with a policy, $\pi$, for any MDP, $\mathcal{M}$, is:
        	\begin{align}
                v^\pi(s) =& \sum_{a \in \mathcal{A}} \pi(a|s) q^\pi(s,a) \label{Eqn:Value-function}\\
                     =& \sum_{a \in \mathcal{A}} \pi(a|s) \sum_{s' \in \mathcal{S}}P(s'|s,a)[\mathcal{R}(s,a) + \gamma v^\pi(s')] \label{Eqn:Value-unrolled}.
            \end{align}
        %
        %$\forall t, \,r =  \mathbf{E}[R(s,a)| S_t=s, A_t=a]$ here.
        %
        $G$ is used to denote $[\mathcal{R}(s,a) + \gamma v^\pi(s')]$ hereafter.
        Re-arranging terms in the Bellman equation, 
            \begin{align}
                v^\pi(s) =& \sum_{a \in \mathcal{A}} \sum_{s' \in \mathcal{S}} \pi(a|s) \frac{P(s',s,a)}{P(s,a)}G \\
                     =& \sum_{a \in \mathcal{A}} \sum_{s' \in \mathcal{S}} \pi(a|s) \frac{P(s',s,a)}{\pi(a|s)P(s)}G \\
                    =& \sum_{a \in \mathcal{A}} \sum_{s' \in \mathcal{S}} \frac{P(s',s,a)}{P(s)}G \\        
                    =& \sum_{a \in \mathcal{A}} \sum_{s' \in \mathcal{S}} \frac{P(a|s,s')P(s,s')}{P(s)}G. \label{Eqn:before-approx}
            \end{align}
         Using the law of total probability, we introduce a new variable $e$ such that:\footnote{Note that $a$ and $e$ are from a joint distribution over a discrete and a continuous random variable. For simplicity, we avoid measure-theoretic notations to represent its joint probability. }
            \begin{align}
                v^\pi(s) =& \sum_{a \in \mathcal{A}} \sum_{s' \in \mathcal{S}}\int_e \!\! \frac{P(a,e|s,s')P(s,s')}{P(s)}G \,  \mathrm{d}e \label{Eqn:latent-space-introduced}.
            \end{align}
         After multiplying and dividing by $P(e|s)$, we have:
             \begin{align}
                    v^\pi(s) =& \sum_{a \in \mathcal{A}} \sum_{s' \in \mathcal{S}}\int_e  \!\!P(e|s) \frac{P(a,e|s,s')P(s,s')}{P(e|s)P(s)}G\, \mathrm{d}e\\
                    =& \sum_{a \in \mathcal{A}} \int_e \!\! P(e|s)  \sum_{s' \in \mathcal{S}}\frac{P(a,e|s,s')P(s,s')}{P(s,e)}G\,\mathrm{d}e \label{Eqn:embedding-space} \\
                    =& \sum_{a \in \mathcal{A}} \int_e \!\! P(e|s)  \sum_{s' \in \mathcal{S}}\frac{P(a,e,s,s')}{P(s,e)}G\,\mathrm{d}e \\
                   =& \sum_{a \in \mathcal{A}} \int_e \!\! P(e|s) \sum_{s' \in \mathcal{S}}P(s',  a|s, e)G\,\mathrm{d}e \\
                    =& \sum_{a \in \mathcal{A}} \int_e \!\! P(e|s) \sum_{s' \in \mathcal{S}}P(s'|s,  a,  e)P(a|s, e)G\,\mathrm{d}e \label{Eqn:embedding-transition}.
            \end{align}
        Since the transition to the next state, $S_{t+1}$, is conditionally independent of $E_t$, given the previous state, $S_t$, and the action taken, $A_t$,
            \begin{align}
                    v^\pi(s) =& \sum_{a \in \mathcal{A}} \int_e \!\! P(e|s) \sum_{s' \in \mathcal{S}}P(s'|s, a)P(a|s, e)G\,\mathrm{d}e \label{Eqn:env-independence}.
            \end{align}
        %
        %Using the Markov property \textcolor{blue}{[GT: You mean the conditional independence Assumption \ref{ass:A1}?]}, 
        Similarly, using the Markov property, action $A_t$ is conditionally independent of $S_t$   given $E_t$, 
            \begin{align}
                v^\pi(s) =& \sum_{a \in \mathcal{A}} \int_e \!\! P(e|s) \sum_{s' \in \mathcal{S}}P(s'|s, a)P(a|e)G\,\mathrm{d}e \label{Eqn:state-independence}.
            \end{align}
        As $P(a|e)$ evaluates to $1$ for representations, $e$, that map to $a$ and $0$ for others (Assumption \ref{ass:A2}),  %we have:        %
            \begin{align}
                        v^\pi(s)=& \sum_{a \in \mathcal{A}} \int_{f^{-1}(a)} P(e|s) \sum_{s' \in \mathcal{S}} P(s'|s, a) G\,\mathrm{d}e \label{Eqn:Indicator}\\
                        =& \sum_{a \in \mathcal{A}} \int_{f^{-1}(a)} P(e|s)  {\hspace{3pt}} q^\pi(s, a)\,\mathrm{d}e. \label{Eqn:last-step}
            \end{align}
         In \eqref{Eqn:last-step}, note that the probability density, $P(e|s)$ is the internal policy, $\pi_i(e|s)$. 
         Therefore,
         \begin{align}
              v^\pi(s)=& \sum_{a \in \mathcal{A}} \int_{f^{-1}(a)} \pi_i(e|s)  {\hspace{3pt}} q^\pi(s, a)\,\mathrm{d}e.
         \end{align}
        \end{proof}

\section{Proof of Lemma 2}
\label{Appendix:lemma-2}
      \begin{lemma}
          For all deterministic functions, $f$, which map each point, $e \in \mathbb{R}^{ d}$, in the representation space to an action, $a \in \mathcal{A}$, the expected updates to $\theta$ based on $\frac{\partial J_i(\theta)}{\partial \theta}$ are equivalent to updates based on $\frac{\partial J_o(\theta,f)}{\partial \theta}$. 
    	That is,
    	\begin{align*}
        	 \frac{\partial J_o(\theta,f)}{\partial \theta} = \frac{\partial J_i(\theta)}{\partial \theta}.
    	\end{align*}
    % 	\label{prop:local-policy-gradient}
    	\end{lemma}
    	\begin{proof}
    	Recall from \eqref{eqn:optimal-policy} that the probability of an action given by the overall policy, $\pi_o$, is
    	\begin{align}
            \pi_o(a|s) &\coloneqq \int_{f^{-1}(a)}\pi_i(e|s)\,\mathrm{d}e.
        \end{align}
    	Using Lemma 1, we express the performance function of the overall policy, $\pi_o$, as:
    	\begin{align}
    	    J_o(\theta,f) =& \sum_{s \in \mathcal{S}} d_0(s) v^{\pi_o}(s)\\  
    	    =& \sum_{s \in \mathcal{S}} d_0(s) \sum_{a \in \mathcal{A}} \int_{f^{-1}(a)}\pi_i(e|s) q^{\pi_o}(s, a)\mathrm{d}e. 
    	\end{align}
        The gradient of the performance function is therefore
    	\begin{align}
                 \frac{\partial J_o(\theta,f)}{\partial \theta} =& \frac{\partial}{\partial \theta}\left[ \sum_{s \in \mathcal{S}} d_0(s) \sum_{a \in \mathcal{A}} \int_{f^{-1}(a)} \pi_i(e|s) q^{\pi_o}(s, a)\mathrm{d}e \right]. \label{Eqn:overall-policy-gradient}			
    	\end{align}
        Using the policy gradient theorem \cite{sutton2000policy} for the overall policy, $\pi_o$,  the partial derivative of $J_o(\theta,f)$ w.r.t.~ $\theta$ is,

    	\begin{align}
        	 \frac{\partial J_o(\theta,f)}{\partial \theta} =&  \sum_{t=0}^{\infty}\mathbf{E}\left[ \sum_{a \in \mathcal{A}} \gamma^t q^{\pi_o}(S_t, a) \frac{\partial}{\partial \theta} \left(\int_{f^{-1}(a)}\!\!\!\!\!\!\!  \pi_i(e|S_t) \mathrm{d}e\right) \right] \\
        	 =& \sum_{t=0}^{\infty}\mathbf{E}\left[  \sum_{a \in \mathcal{A}} \gamma^t \int_{f^{-1}(a)} \frac{\partial}{\partial \theta} \left(\pi_i(e|S_t)\right) q^{\pi_o}(S_t, a) \mathrm{d}e \right]  \label{Eqn:swapped-gradient} \\
        	 =& \sum_{t=0}^{\infty} \mathbf{E}\left[ \sum_{a \in \mathcal{A}} \gamma^t \int_{f^{-1}(a)} \!\!\!\!\!\!\!\!\!\!\pi_i(e|S_t)  \frac{\partial}{\partial \theta} \ln\left(\pi_i(e|S_t)\right) 
        	 q^{\pi_o}(S_t,a)\mathrm{d}e \right] \label{Eqn:score-fn}. 
    	\end{align}
    	Note that since $e$ deterministically maps to $a$, $q^{\pi_o}(S_t,a) = q^{\pi_i}(S_t,e)$. Therefore,
    	%Since each $e$ gets mapped to only a unique $a$ by the function $f$, the summation over $a$ and an inner integral over $f^{-1}(a)$ can be replaced by an integral over the entire domain of $e$. 
    	%
    	\begin{align}
    	 	 \frac{\partial J_o(\theta,f)}{\partial \theta} =& \sum_{t=0}^{\infty} \mathbf{E}\left[  \gamma^t \sum_{a \in \mathcal{A}}  \int_{f^{-1}(a)} \!\!\!\!\!\!\!\!\!\!\pi_i(e|S_t)  \frac{\partial}{\partial \theta} \ln\left(\pi_i(e|S_t)\right)  q^{\pi_i}(S_t,e)\mathrm{d}e\right]. \label{Eqn:no-a}  
    	\end{align}
    	Finally, since each $e$ is mapped to a unique action by the function $f$, the nested summation over $a$ and its inner integral over $f^{-1}(a)$ can be replaced by an integral over the entire domain of $e$. Hence,
    	\begin{align}
    	\label{eqn:gradient-equivalence}
    	   \frac{\partial J_o(\theta,f)}{\partial \theta} =& \sum_{t=0}^{\infty}\mathbf{E}\left[  \gamma^t \int_e \! \pi_i(e|S_t) \frac{\partial}{\partial \theta} \ln\left(\pi_i(e|S_t)\right)  q^{\pi_i}(S_t, e)\mathrm{d}e\right] \\
    	    =& \sum_{t=0}^{\infty}\mathbf{E}\left [  \gamma^t \int_e q^{\pi_i}(S_t, e) \frac{\partial}{\partial \theta} \pi_i(e|S_t)  \, \mathrm{d}e\right ] \\
        	=&  \frac{\partial J_i(\theta)}{\partial \theta}. 
    	\end{align}
    	\end{proof}

\section{Convergence of PG-RA}
\label{Appendix:three-timescale}
     	
        To analyze the convergence of PG-RA, we first briefly review existing two-timescale convergence results for actor-critics.
        Afterwards, we present a general setup for stochastic recursions of three dependent parameter sequences.
        Asymptotic behavior of the system is then discussed using three different timescales, by adapting existing multi-timescale results by \citet{borkar2009stochastic}.
        This lays the foundation for our subsequent convergence proof. 
        Finally, we prove convergence of the PG-RA method, which extends standard actor-critic algorithms using a new action prediction module, using a three-timescale approach.
        This technique for the proof is not a novel contribution of the work.
        We leverage and extend the existing convergence results of actor-critic algorithms \cite{borkar1997actor} for our algorithm. 

	\subsection{Actor-Critic Convergence Using Two-Timescales }
	In the actor-critic algorithms, the updates to the policy depends upon a critic that can estimate the value function associated with the policy at that particular instance. 
	One way to get a good value function is to fix the policy temporarily and update the critic in an inner-loop that uses the transitions drawn using only that fixed policy.
	While this is a sound approach, it requires a possibly large time between successive updates to the policy parameters and is severely sample-inefficient. 
	Two-timescale stochastic approximation methods \cite{bhatnagar2009natural,konda2000actor} circumvent this difficulty.  
	The faster update recursion for the critic ensures that asymptotically it is always a close approximation to the required value function before the next update to the policy is made.

	\subsection{Three-Timescale Setup}
	In our proposed algorithm, to update the action prediction module, one could have also considered an inner loop that uses transitions drawn using the fixed policy for supervised updates. 
	Instead, to make such a procedure converge faster, we extend the existing two-timescale actor-critic results and take a three-timescale approach. %extend the two-time scale analysis using three different step-size schedules.
	%
	%The key constructs of the proof are based on the work of \citet{bhatnagar2009natural} and \citet{konda2000actor}.  
	
    Consider the following system of stochastic ordinary differential equations (ODE): 
	\begin{align}
	    X_{t+1} =& X_{t} + \alpha_t^x (F_x(X_{t}, Y_{t}, Z_{t}) + \mathcal{N}_{t+1}^1), \label{Eqn:critic-recursion} \\
	    Y_{t+1} =& Y_{t} + \alpha_t^y (F_y(Y_{t}, Z_{t}) + \mathcal{N}_{t+1}^2), \label{Eqn:representation-recursion}\\
	    %Z_{t+1} =& Z_{t} - \alpha_t^z (h(X_{t}, Z_{t}) + \mathcal{N}_t^3),
	    Z_{t+1} =& Z_{t} + \alpha_t^z (F_z(X_{t}, Y_{t}, Z_{t}) + \mathcal{N}_{t+1}^3),\label{Eqn:policy-recursion}
	\end{align}
	where, $F_x,F_y$ and $F_g$ are Lipschitz continuous functions and $\{\mathcal{N}_t^1\}, \{\mathcal{N}_t^2\}$, $\{\mathcal{N}_t^3\}$ are the associated martingale difference sequences for noise w.r.t. the increasing $\sigma$-fields $\mathcal{F}_t$ = $\sigma (X_n, Y_n, Z_n, \mathcal{N}_n^1, \mathcal{N}_n^2, \mathcal{N}_n^3, n \leq t), t \geq 0$, satisfying
	$$\mathbf{E}[||N^i_{t+1}||^2|\mathcal{F}_t] \leq D_1(1 + ||X_t||^2 + ||Y_t||^2 + ||Z_t||^2),$$ 
	for $i=1,2,3$, $t \geq 0$ and any constant $D < \infty$ such that the quadratic variation of noise is always bounded.
	To study the asymptotic behavior of the system, consider the following standard assumptions,
	\begin{assumpB}[Boundedness]
	\renewcommand\theass{A\arabic{ass}}
	    \label{ass:bounded}
	    $\underset{t}{\mathrm{sup}} \ (|| X_t|| + || Y_t|| + || Z_t||) < \infty$, almost surely.
	\end{assumpB}
	\begin{assumpB}[Learning rate schedule]
	    \label{ass:learning-rate}
	    The learning rates $\alpha_t^x, \alpha_t^y$ and $\alpha_t^z$ satisfy:
		\begin{align}
    	    \sum_t \alpha_t^x= \infty, \sum_t \alpha_t^y = \infty,  \sum_t \alpha_t^z = \infty, \\
    	    \sum_t (\alpha^x_t)^2< \infty, \sum_t (\alpha^y_t)^2< \infty, \sum_t (\alpha^z_t)^2 < \infty, \\
    	     \text{As} \hspace{5pt} t \rightarrow \infty, \hspace{10pt} \frac{\alpha^z_t}{\alpha^y_t} \rightarrow 0, \  \frac{\alpha^y_t}{\alpha^x_t} \rightarrow 0.   \label{Eqn:step-schedule} 
	    \end{align}    
	\end{assumpB}
	\begin{assumpB}[Existence of stationary point for Y]
	    \label{ass:eqbm-Y}
    	 The following ODE has a globally asymptotically stable equilibrium $\mu_1(Z)$, where $\mu_1(\cdot)$ is a Lipschitz continuous function.
        \begin{align}
            \dot Y =& F_y(Y(t), Z) \label{Eqn:fixed-rep-alternate}
        \end{align}
	\end{assumpB}
	\begin{assumpB}[Existence of stationary point for X]
	    \label{ass:eqbm-X}
	    The following ODE has a globally asymptotically stable equilibrium $\mu_2(Y, Z)$, where $\mu_2(\cdot, \cdot)$ is a Lipschitz continuous function.
	    \begin{align}
            \dot X =& F_x(X(t),Y, Z), \label{Eqn:fixed-critic}
        \end{align}
	\end{assumpB}
	\begin{assumpB}[Existence of stationary point for Z]
	    \label{ass:eqbm-Z}
	    The following ODE has a globally asymptotically stable equilibrium $Z^\star$,
        \begin{align}
        %   \dot Z =& h(\mu_2(\mu_1(Z(t)),Z(t)), Z(t)) \label{Eqn:fixed-critic-optimal}
        \dot Z =& F_z(\mu_2(\mu_1(Z(t)),Z(t)), \mu_1(Z(t)), Z(t)). \label{Eqn:fixed-critic-optimal}
        \end{align}
	\end{assumpB}
    
%    {\color{red} describe what these assumptions mean, before stating the property?}
    Assumptions \ref{ass:bounded}--\ref{ass:learning-rate} are required to bound the values of the parameter sequence and make the learning rate well-conditioned, respectively. 
    Assumptions \ref{ass:eqbm-Y}-\ref{ass:eqbm-X} ensure that there exists a global stationary point for the respective recursions, individually, when other parameters are held constant.
    Finally, Assumption \ref{ass:eqbm-Z} ensures that there exists a global stationary point for the update recursion associated with $Z$, if between each successive update to $Z$, $X$ and $Y$ have converged to their respective stationary points.

    \begin{lemma}
        \label{prop:convergence}
        Under Assumptions \ref{ass:bounded}-\ref{ass:eqbm-Z}, $(X_t,Y_t,Z_t) \rightarrow (\mu_2(\mu_1(Z^\star), Z^\star), \mu_1(Z^\star), Z^\star)$ as $t \rightarrow \infty$, with probability one. 
    \end{lemma}
	\begin{proof}
	We adapt the multi-timescale analysis by \citet{borkar2009stochastic} to analyze the above system of equations using three-timescales.
	First we present an intuitive explanation and then we formalize the results.
	%We extend the two time-scale analysis  \cite{borkar1997actor} using three time-scales \cite{bhatnagar2005adaptive}.
	%
	%We only present a sketch of the proof and discuss the intuition behind the key ideas. For detailed discussions on two and three time-scale convergence, we refer the readers to \cite{konda2000actor,bhatnagar2009natural,bhatnagar2005adaptive}.

	%
    Since these three updates are not independent at each time step, we consider three step-size schedules: $\{\alpha_t^x \}$, $\{\alpha_t^y \}$ and $\{\alpha_t^z \}$, which satisfy Assumption \ref{ass:learning-rate}.
	As a consequence of \eqref{Eqn:step-schedule}, the recursion \eqref{Eqn:representation-recursion} is `faster' than \eqref{Eqn:policy-recursion}, and \eqref{Eqn:critic-recursion} is `faster' than both \eqref{Eqn:representation-recursion} and \eqref{Eqn:policy-recursion}.
	In other words, $Z$ moves on the slowest timescale and the $X$ moves on the fastest. 
	Such a timescale is desirable since $Z_t$ converges to its stationary point if at each time step the value of the corresponding converged $X$ and $Y$ estimates are used to make the next $Z$ update (Assumption \ref{ass:eqbm-Z}).

    To elaborate on the previous points, first consider the ODEs:
    \begin{align}
        \dot Y =& F_y(Y(t), Z(t)) \label{Eqn:fixed-rep-a}, \\
        \dot Z =& 0 \label{Eqn:fixed-policy-a}.
    \end{align}
    Alternatively, one can consider the ODE
    \begin{align}
        \dot Y =& F_y(Y(t), Z), 
    \end{align}
    in place of \eqref{Eqn:fixed-rep-a}, because $Z$ is fixed \eqref{Eqn:fixed-policy-a}. 
	%
%	\textcolor{red}{Need to say the main result: if $Z$ has a fixed point for all $Y$, and we use two time scales, then the whole system converges to the same thing is (5) and (6), where Z is always set to the value that corresponds to the current Y. It is this result that establishes convergence. If this isn't clear, come knock on my door and we can talk quickly.}
	%
	Now, under Assumption \ref{ass:eqbm-Y} we know that the iterative update \eqref{Eqn:representation-recursion} performed on $Y$, with a fixed $Z$, will eventually converge to a corresponding stationary point.

    Now, with this converged $Y$, consider the following ODEs:
    \begin{align}
        \dot X =& F_x(X(t), Y(t), Z(t)), \label{Eqn:critic-recursion-b} \\
        \dot Y =& 0 \label{Eqn:fixed-rep-b}, \\
        \dot Z =& 0 \label{Eqn:fixed-policy-b}.
    \end{align}
    Alternatively, one can consider the ODE
    \begin{align}
        \dot X =& F_x(X(t),Y, Z), 
    \end{align}
    in place of \eqref{Eqn:critic-recursion-b}, as $Y$ and $Z$ are fixed \eqref{Eqn:fixed-rep-b}-\eqref{Eqn:fixed-policy-b}. 
    As a consequence of Assumption \ref{ass:eqbm-X}, $X$ converges when both $Y$ and $Z$ are held fixed.

     Intuitively, as a result of Assumption \ref{ass:learning-rate}, in the limit, the learning-rate, $\alpha_t^z$ becomes very small relative to $\alpha_t^y$.
     This makes $Z$ `quasi-static' compared to $Y$ and has an effect similar to fixing $Z_t$ and running the iteration \eqref{Eqn:representation-recursion} forever to converge at $\mu_1(Z_t)$.
     Similarly, both $\alpha_t^y$ and $\alpha_t^z$ become very small relative to $\alpha_t^x$.
     Therefore, both $Y$ and $Z$ are `quasi-static' compared to $X$, which has an effect similar to fixing $Y_t$ and $Z_t$, and running the iteration \eqref{Eqn:critic-recursion} forever.
     In turn, this makes $Z_t$ see $X_t$ as a close approximation to $\mu_2(\mu_1(Z(t)),Z(t))$ always, and thus $Z_t$ converges to $Z^\star$  due to Assumption \ref{ass:eqbm-Z}.  
     %
     
    %In Theorem 2, we showed the equivalence between the actual performance gradient $\nabla J_a$ and $\nabla J_e$ in the absence of critic.
    %
    
    Formally, define three real-valued sequences $\{i_t\}, \{j_t\}$ and $\{k_t\}$ as $i_t=\sum_{n=0}^{t-1} \alpha_t^y, j_t=\sum_{n=0}^{t-1} \alpha_t^x$ and $k_t=\sum_{n=0}^{t-1} \alpha_t^z$, respectively.
    These are required for tracking the continuous time ODEs, in the limit, using discretized time.
    Note that $(i_t-i_{t-1}), (j_t-j_{t-1}), (k_t-k_{t-1})$ almost surely converge to $0$ as $t \rightarrow \infty$. 

    Define continuous time processes $\bar Y(i), \bar Z(i), i \geq 0$ as $\bar Y(i_t) = Y_t, \bar Z(i_t) = Z_t$, respectively with linear interpolations in between. 
    For $s \geq 0$, let $Y^s(i), Z^s(i), i \geq s$ denote the trajectories of \eqref{Eqn:fixed-rep-a}--\eqref{Eqn:fixed-policy-a} with $Y^s(s) = \bar Y(s)$ and $Z^s(s) = \bar Z(s)$. 
    Note that because of \eqref{Eqn:fixed-policy-a}, $\forall_i \geq s \ Z^s(i) = \bar Z(s)$.
    Now consider re-writing  \eqref{Eqn:representation-recursion}--\eqref{Eqn:policy-recursion} as,
    \begin{align}
	    Y_{t+1} =& Y_{t} + \alpha_t^y (F_y(Y_{t}, Z_{t}) + \mathcal{N}_{t+1}^2),\\
        Z_{t+1} =& Z_{t} + \alpha_t^y \left( \frac{\alpha_t^z}{\alpha_t^y} (F_z(X_{t}, Y_{t}, Z_{t}) + \mathcal{N}_{t+1}^3) \right).
    \end{align}
    When the time discretization corresponds to $\{i_t\}$, this shows that \eqref{Eqn:representation-recursion}--\eqref{Eqn:policy-recursion} can be seen as `noisy' Euler discretizations of the ODE \eqref{Eqn:fixed-rep-alternate} (or, equivalently of ODEs \eqref{Eqn:fixed-rep-a}--\eqref{Eqn:fixed-policy-a}), but as $\dot Z = 0$ this ODE has an approximation error of $\frac{\alpha_t^z}{\alpha_t^y} (F_z(X_{t}, Y_{t}, Z_{t}) + \mathcal{N}_{t+1}^3)$.
    However, asymptotically, this error vanishes as $\frac{\alpha^z_t}{\alpha^y_t} \rightarrow 0$. 
    Now using results by \citet{borkar2009stochastic}, it can be shown that, for any given $T \geq 0$, as $s \rightarrow \infty$,
    \begin{align}
    \underset{i \in [s,s+T]}{\mathrm{sup}} \ || \bar Y(i) - Y^s(i)||  \rightarrow 0, \\
    \underset{i \in [s,s+T]}{\mathrm{sup}} \ || \bar Z(i) - Z^s(i)||  \rightarrow 0, \label{Eqn:first-conevrgence}
    \end{align}
    with probability one.
    Hence, in the limit, the discretization error also vanishes and $(Y(t), Z(t)) \rightarrow (\mu_1(Z(t)), Z(t))$. Similarly for \eqref{Eqn:critic-recursion}--\eqref{Eqn:policy-recursion}, with $\{j_t\}$ as time discretization, and using the fact that both $\frac{\alpha^z_t}{\alpha^y_t} \rightarrow 0$, and $\frac{\alpha^y_t}{\alpha^x_t} \rightarrow 0$, a noisy Euler discretization can be obtained for ODE \eqref{Eqn:fixed-critic} (or equivalently for ODEs \eqref{Eqn:critic-recursion-b}--\eqref{Eqn:fixed-policy-b}). 
    Hence, in the limit, $(X(t), Y(t), Z(t)) \rightarrow (\mu_2(\mu_1(Z(t)),Z(t)), \mu_1(Z(t)), Z(t))$.  

    Now consider re-writing \eqref{Eqn:policy-recursion} as:
    \begin{align}
        Z_{t+1} =& Z_{t} \\
        & + \alpha_t^z (F_z(\mu_2(\mu_1(Z(t)),Z(t)), \mu_1(Z(t)), Z(t)))\\
        & - \alpha_t^z (F_z(\mu_2(\mu_1(Z(t)),Z(t)), \mu_1(Z(t)), Z(t))) \\
        & + \alpha_t^z (F_z(X_{t}, Y_{t}, Z_{t})) \\ 
        & + \alpha_t^z (\mathcal{N}_{t+1}^3) . \label{Eqn:noisy-eqbm-Z}
    \end{align}   
    This can be seen as a noisy Euler discretization of the ODE \eqref{Eqn:fixed-critic-optimal}, along the time-line $\{k_t\}$, with the error corresponding to the third, fourth and fifth terms on the RHS of \eqref{Eqn:noisy-eqbm-Z}.
    We denote these error terms as $I, II$ and $III$, respectively.
    In the limit, using the result that $(X(t), Y(t), Z(t)) \rightarrow (\mu_2(\mu_1(Z(t)),Z(t)), \mu_1(Z(t)), Z(t))$ as $t \rightarrow \infty$, the error $I + II$ vanishes.
    Similarly, martingale noise error, $III$, vanishes asymptotically as a consequence of bounded $Z$ values and $\sum_t (\alpha^z_t)^2 < \infty$. 
    Now using sequence of approximations using Gronwall's inequality, it can be shown that \eqref{Eqn:noisy-eqbm-Z} converges to $Z^*$ asymptotically \cite{borkar2009stochastic}.

    Therefore, under the Assumptions \ref{ass:bounded}-\ref{ass:eqbm-Z}, $(X_t,Y_t,Z_t) \rightarrow (\mu_2(\mu_1(Z^\star), Z^\star), \mu_1(Z^\star), Z^\star)$ as $t \rightarrow \infty$. 
    %When non-linear function approximators are used, this system converges to a local optimum.  
    \end{proof}

	\subsection{PG-RA Convergence Using Three- Timescales: }
	\label{Appendix:PG-RA-convergence}

	%

    %
    %the internal policy's parameters, $Z_t$, converge to $\{ z \in \mathcal{Z} | \nabla J_a(z) = 0 \}$;
    %
    %In parallel, the action representation component's parameters, $Y$, and the critic's parameters, $X$, converge to the corresponding action probabilities and state-value estimates, respectively. 
	%

	%Let $\omega$ correspond to the paramters of critic.
	Let the parameters of the critic and the internal policy be denoted as $\omega$ and $\theta$ respectively. Also, let  $\phi$ denote all the parameters of $\hat f$ and $\hat g$. 
    Similar to prior work \cite{bhatnagar2009natural,degris2012off,konda2000actor}, for analysis of the updates to the parameters, we consider the following standard assumptions required to ensure existence of gradients and bound the parameter ranges.
	%Then, a locally optimal set of parameters for achieving the maximum expected discounted return is defined to be $\omega^\star, \phi^\star$ and $\theta^\star$.
	%
	%For the internal policy, the optimal parameters  $\theta^\star$, maximizes the overall performance function and the optimal parameters, $\phi^\star$, provide the best approximation to Assumption A1. 
	%
	%Similarly, an optimal critic is paramterized by $\omega^\star$ that provides the best estimate of the value function associated with this optimal overall policy.
	\begin{assumpA}
	    \label{ass:differentiable}
	    For any state action-representation pair (s,e), internal policy, $\pi_i(e|s)$, is continuously differentiable in the parameter  $\theta$. 
	\end{assumpA}
	\begin{assumpA}
	\label{ass:projection}
	 The updates to the parameters, $\theta \in \mathbb{R}^{d_\theta}$, of the internal policy, $\pi_i$, includes a projection operator $\Gamma : \mathbb{R}^{d_\theta} \rightarrow \mathbb{R}^{d_\theta}$ that projects any $x \in \mathbb{R}^{d_\theta}$ to a compact set $\mathcal{C} = \{x|c_i(x) \leq 0, i=1,...,n\} \subset \mathbb{R}^{d_\theta}$, where $c_i(\cdot), i=1,...,n$ are real-valued, continuously differentiable functions on $\mathbb{R}^{d_\theta}$ that represents the constraints specifying the compact region. For each $x$ on the boundary of $\mathcal C$, the gradients of the active $c_i$ are considered to be linearly independent.  
	\end{assumpA}
	\begin{assumpA}
        \label{ass:param-bounded}
        The iterates $\omega_t$ and $\phi_t$ satisfy  $\underset{t}{\mathrm{sup}} \ (|| \omega_t||) < \infty$ and $\underset{t}{\mathrm{sup}} \ (|| \phi_t||) < \infty$.
	\end{assumpA}
	Let $v(\cdot)$ be the gradient vector field on $\mathcal{C}$. We define another vector field operator $\hat \Gamma$,
	\begin{align}
	    \hat \Gamma(v(\theta)) &\coloneqq \lim_{h\rightarrow0} \frac{\Gamma(\theta + hv(\theta)) - \theta}{h},
	\end{align}
    that projects any gradients leading outside the compact region,  $\mathcal{C}$, back to $\mathcal{C}$.
    \begin{thm}
    % 	\label{thm:convergence}
    	  Under Assumptions (\ref{ass:A1})-(\ref{ass:param-bounded}), the internal policy parameters   $\theta_t$ converge to $\mathcal{\hat Z} = \left\{x \in \mathcal{C}|\hat \Gamma\left(\frac{\partial J_i(x)}{\partial \theta}\right)=0\right\}$ as $t \rightarrow \infty$, with probability one.
	    \end{thm}
	\begin{proof} 
	PG-RA algorithm considers the following stochastic update recursions for the critic, action representation modules, and the internal policy, respectively:
	\begin{align}
	    \omega_{t+1} =& \omega_{t} + \alpha_t^{\omega} \delta_t \frac{\partial v(s)}{\partial \omega}  \label{Eqn:PGRA-critic-recursion} \\
	    \phi_{t+1} =& \phi_{t} + \alpha_t^{\phi} \frac{- \partial \log \hat P(a|s,s')}{\partial \phi} \label{Eqn:PGRA-representation-recursion}\\
	    %Z_{t+1} =& Z_{t} - \alpha_t^z (h(X_{t}, Z_{t}) + \mathcal{N}_t^3),
	    \theta_{t+1} =& \theta_{t} + \alpha_t^{\theta} \hat\Gamma\left(\delta_t\frac{\partial \log \pi_i(e|s)}{\partial \theta}\right),\label{Eqn:PGRA-policy-recursion}
	\end{align}
	where, $\delta_t$ is the TD-error and is given by:
	\begin{align}
	    \delta_t =& r + \gamma v(s') - v(s).
	\end{align}
	We now establish how these updates can be mapped to the three ODEs \eqref{Eqn:critic-recursion}--\eqref{Eqn:policy-recursion} satisfying Assumptions (\ref{ass:bounded})--(\ref{ass:eqbm-Z}), so as to leverage the result from Lemma \ref{prop:convergence}. 
	To do so, we must consider how the recursions are dependent on each other.
	Since the reward observed is a consequence of the action executed using the internal policy and the action representation module, it makes $\delta$ dependent on both $\phi$ and $\theta$.
	Due to the use of bootstrapping and a baseline, $\delta$ is also dependent on $\omega$.
	As a result, the updates to both $\omega$ and $\theta$ are dependent on all three sets of parameters.
	In contrast, notice that the updates to the action representation module is independent of the rewards/critic and is thus dependent only on $\theta$ and $\phi$.
	Therefore, the ODEs that govern the update recursions for PG-RA parameters are of the form \eqref{Eqn:critic-recursion}--\eqref{Eqn:policy-recursion}, where $(\omega, \phi, \theta)$ correspond directly to $(X,Y,Z)$.
	The functions $F_x,F_y$ and $F_z$ in  \eqref{Eqn:critic-recursion}--\eqref{Eqn:policy-recursion} correspond to the semi-gradients of TD-error, gradients of self-supervised loss, and policy gradients, respectively. 
    %
    %It can be observed that $\phi^\star = \mu_1(\theta^\star)$ and $\omega_t^\star=\mu_2(\phi^\star, \theta^\star)$. 
    %
	Lemma \ref{prop:convergence} can now be leveraged if the associated assumptions are also satisfied by our PG-RA algorithm.

	\textbf{For requirement \ref{ass:bounded}}, as a result of the projection operator $\Gamma$, the internal policy parameters, $\theta$, remain bounded. Further, by assumption, $\omega$ and $\phi$ always remain bounded as well. Therefore, we have that $\underset{t}{\mathrm{sup}} \ (|| \omega_t|| + || \phi_t|| + || \theta_t||) < \infty$.

		\textbf{For requirement \ref{ass:learning-rate}}, the learning rates $\alpha_t^\omega, \alpha_t^\phi$ and $\alpha_t^\theta$  are hyper-parameters and can be set such that as $t \rightarrow \infty$,
		$$ \frac{\alpha_t^\theta}{\alpha_t^\phi} \rightarrow 0, \frac{\alpha_t^\phi}{\alpha_t^\omega} \rightarrow 0,$$
		to meet the three-timescale requirement in Assumption \ref{ass:learning-rate}. 

	\textbf{For requirement \ref{ass:eqbm-Y}}, recall that when the internal policy has fixed parameters, $\theta$, the updates to the action representation component follows a supervised learning procedure. 
	For linear parameterization of estimators $\hat f$ and $\hat g$, the action prediction module is equivalent to a bi-linear neural network.
	%
	%This is a special case of linear neural network models.
	%
	Multiple works have established that for such models, there are no spurious local minimas and the Hessian at every saddle point has at least one negative eigenvalue  \cite{kawaguchi2016deep,haeffele2017global,zhu2018global}. 
	Further, the global minima can be achieved by stochastic gradient descent.
	This ensures convergence to the required critical point and satisfies Assumption \ref{ass:eqbm-Y}.
	%
		%From an ODE perspective,  the convergence to the global fixed point for the linear system 
		
		%{\color{red} Even product of two linear matrices is non-linear. Linearizing it will only guarantee local max. Use biconvexity and alternate gradient descent to show convergence to optimal.} depends on the eigen values of the matrix $L$,
		%$$L = (L_{ij}) = \frac{\partial F_y^i}{\partial \phi_j} $$
		%corresponding to the Jacobian of the function $F_y =  \nabla_\phi \log \bar P(a|s,s')$ which governs the evolution of the ODE. 
		%
		%$L$ corresponds to the the Hessian of the objective function here.
		%Since it is a positive semi-definite matrix and has all non-negative real eigen values, 
		%
        %using the Hartman-Grobman theorem it can be seen that this ODE converges to its global fixed point $\mu_1(\theta)$. This satisfies Assumption \ref{ass:eqbm-Y}
	%
	%
	
	\textbf{For requirement \ref{ass:eqbm-X}}, given a fixed policy (fixed action representations and fixed internal policy) the proof of convergence of a linear critic to the stationary point $\mu_2(\phi, \theta)$ using TD($\lambda$) is a well established result \cite{tsitsiklis1996analysis}. %
	We use  $\lambda = 0$ in our algorithm, the proof however carries through for $\lambda > 0$ as well.
	This satisfies Assumption \ref{ass:eqbm-X}.

	\textbf{For requirement \ref{ass:eqbm-Z}}, the condition can be relaxed to a local rather than global asymptotically stable fixed point, because we only need convergence.
 Under the optimal critic and action representations for every step, the internal policy follows its internal policy gradient. 
 Using Lemma \ref{prop:local-policy-gradient}, we established that this is equivalent to following the policy gradient of the overall policy and thus the internal policy converges to its local fixed point as well.

	This completes the necessary requirements, the remaining proof now follows from Lemma \ref{prop:convergence}.
	\end{proof}

	\section{Implementation Details } 
	\label{Appendix:parameter}
		\subsection{Parameterization}
	     In our experiments, we consider a parameterization that minimizes the computational complexity of the algorithm.
	     Learning the parameters of the action representation module, as in  \eqref{Eqn:self-supervised-loss}, requires computing the value $\hat P(a|s,s')$ in \eqref{eqn:action-rep-estimator}.
	     This involves a complete integral over $e$. 
        Due to the absence of any closed form solution, we need to rely on a stochastic estimate. 
        Depending on the dimensions of $e$, an extensive sample based evaluation of this expectation can be computationally expensive. 
        To make this more tractable, we approximate \eqref{eqn:action-rep-estimator} by mean-marginalizing it using the estimate of the mean from $\hat g$. 
    	That is, we approximate \eqref{eqn:action-rep-estimator}
        as $\hat f (a|\hat g(s,s'))$.
        We then parameterize $\hat f (a|\hat g(s,s'))$ as,
        \begin{align}
            && \hat f(a|\hat g(s,s')) =& \frac{\mathbf{e}^{z_a/\tau}}{\sum_{a'}\mathbf{e}^{z_{a'}/\tau}} \label{Eqn:Estimated-dist}, &\\
            \text{where,} &&
            z_a =& W^\top_a \hat g(s,s') \label{Eqn:dot}.
        \end{align}
        This estimator, $\hat f$, models the probability  of any action, $a$, based on its similarity with a given representation $e$. 
        In \eqref{Eqn:dot}, $W \in \mathbb{R}^{d_e \times |\mathcal{A}|}$ is a matrix where each column represents a learnable action representation of dimension $\mathbb{R}^{d_e}$.
        $W^\top_a$ is the transpose of the vector corresponding to the representation of the action $a$, and $z_a$ is its measure of similarity with the embedding from $\hat g(s,s')$.  
        % The similarity which is used to compute a similarity score, $z_a$.
    	%
        To get valid probability values, a Boltzmann distribution is used with $\tau$ as a temperature variable.
        In the limit when $\tau \rightarrow 0$ the conditional distribution over actions becomes the required deterministic estimate for $\hat f$. 
        That is, the entire probability mass would be on the action, $a$, which has the most similar representation to $e$.
        To ensure empirical stability during training, we relax $\tau$ to $1$.
        During execution, the action, $a$, which has the most similar representation to $e$, is chosen for execution.
        In practice, the linear decomposition in \eqref{Eqn:dot} is not restrictive as $\hat g$ can still be any differentiable function approximator, like a neural network.

	\subsection{Hyper-parameters}
	For the maze domain, single layer neural networks were used to parameterize both the actor and critic, and the learning rates were searched over $\{1e-2, 1e-3, 1e-4, 1e-5\}$. State features were represented using the $3^\text{rd}$ order coupled Fourier basis \cite{konidaris2011value}. The discounting parameter $\gamma$ was set to $0.99$ and $\lambda$ to $0.9$. Since it was a toy domain, the dimensions of action representations were fixed to $2$. $2000$ randomly drawn trajectories were used to learn an intial representation for the actions. Action representations were only trained once in the beginning and kept fixed from there on.
	
	For the real-world environments, $2$ layer neural networks were used to parameterize both the actor and critic, and the learning rates were searched over $\{1e-2, 1e-3, 1e-4, 1e-5\}$. Similar to prior work, the module for encoding state features was shared to reduce the number of parameters, and the learning rate for it was additionally searched over $\{1e-2, 1e-3, 1e-4, 1e-5\}$. The dimension of the neural network's hidden layer was searched over $\{64, 128, 256\}$.  The discounting parameter $\gamma$ was set to $0.9$. For actor-critic based results $\lambda$ was set to $0.9$ and for DPG the target actor and policy update rate was fixed to its default setting of $0.001$. The dimension of action representations were searched over $\{16, 32, 64\}$. Initial $10,\!000$ randomly drawn trajectories were used to learn an initial representation for the actions. The action prediction component was continuously improved on the fly, as given in PG-RA algorithm. 
	
	For all the results of the PG-RA based algorithms, since $\pi_i$ was defined over a continuous space, it was parameterized as the isotropic normal distribution. The value for variance was searched over $\{0.1, 0.25, 1, -1\}$, where $-1$ represents learned variance. Function $\hat g$ was parameterized to concatenate the state features of both $s$ and $s'$ and project to the embedding space using a single layer neural network with \textit{Tanh} non-linearity. To keep the effective range of action representations bounded, they were also transformed by \textit{Tanh} non-linearity before computing the similarity scores. Though the similarity metric is naturally based on the dot product, other distance metrics are also valid. We found squared Euclidean distance to work best in our experiments. The learning rates for functions $\hat f$ and $\hat g$ were jointly searched over $\{1e-2, 1e-3, 1e-4\}$. All the results were obtained for $10$ different seeds to get the variance. 
	As our proposed method decomposes the overall policy into two components, the resulting architecture resembles that of a one layer deeper neural network.
	Therefore, for the baselines, we ran the experiments with a hyper-parameter search for policies with additional depths $\{1, 2, 3\}$, each with different combinations of width $\{2, 16, 64\}$.
	The remaining architectural aspects and properties of the hyper-parameter search for the baselines were performed in the same way as mentioned above for our proposed method.
	All the results presented in Figure \ref{Fig:performance-plots} corresponds to the hyper-parameter setting that performed the best.

\end{document}